\documentclass[11pt]{article}

\usepackage[top=1.4in, bottom=1.4in, left=1.2in, right=1.2in]{geometry}
\usepackage[numbers, sort&compress]{natbib}
\usepackage[utf8]{inputenc} %
\usepackage[T1]{fontenc}    %
\usepackage{hyperref}       %
\usepackage{url}            %
\usepackage{booktabs}       %
\usepackage{amsfonts}       %
\usepackage{nicefrac}       %
\usepackage{microtype}      %
\usepackage{xcolor}         %

\title{Learning latent causal graphs via mixture oracles}
\author{Bohdan Kivva\\
bkivva@uchicago.edu\\
University of Chicago
\and Goutham Rajendran\\
goutham@uchicago.edu\\
University of Chicago
\and Pradeep Ravikumar\\
pradeepr@cs.cmu.edu\\
Carnegie Mellon University
\and Bryon Aragam\\
bryon@chicagobooth.edu\\
University of Chicago}

\usepackage{cmap}
\usepackage{textcomp}
\usepackage{comment}

\usepackage{xspace}
\usepackage{enumerate}
\usepackage{enumitem}
\usepackage{amssymb}
\usepackage{amsmath}
\usepackage{mathtools}
\usepackage{multirow}

\usepackage{pgf, tikz}
\usetikzlibrary{arrows, automata}

\usepackage{bm}
\usepackage{bbm}
\usepackage{amsthm}
\usepackage[capitalise]{cleveref}

\usepackage{subcaption}

\usepackage[parfill]{parskip} %

\allowdisplaybreaks

\newcommand{\RR}{\mathbb{R}}
\newcommand{\calF}{\mathcal{F}}

\newcommand{\Gam}{\Gamma}

\newcommand{\pr}{\mathbb{P}}

\newcommand{\R}{\mathbb{R}}

\newcommand{\given}{\,|\,}  %

\DeclareMathOperator{\Span}{span}

\DeclareMathOperator{\pa}{pa} %
\DeclareMathOperator{\ch}{ch} %

\newcommand\indep{\independent}
\newcommand\independent{\protect\mathpalette{\protect\independenT}{\perp}}
\def\independenT#1#2{\mathrel{\rlap{$#1#2$}\mkern4mu{#1#2}}}

\newtheorem{theorem}{Theorem}
\newtheorem{corollary}[theorem]{Corollary}
\newtheorem{example}[theorem]{Example}
\newtheorem{lemma}[theorem]{Lemma}

\newtheorem{definition}[theorem]{Definition}

\newtheorem{assumption}[theorem]{Assumption}
\newtheorem{observation}[theorem]{Observation}
\newtheorem{problem}[theorem]{Problem}
\newtheorem{remark}[theorem]{Remark}

\numberwithin{theorem}{section}

\newcommand{\gr}{G}
\newcommand{\ver}{V}
\newcommand{\edg}{E}
\newcommand{\obs}{X}
\newcommand{\lat}{H}
\newcommand{\nhbd}{S}

\newcommand{\mix}{Z}
\newcommand{\nobs}{n}
\newcommand{\nlat}{m}
\newcommand{\ncomp}{k}
\newcommand{\nmix}{K}
\newcommand{\bipartite}{\Gamma}
\newcommand{\latentgr}{\Lambda}
\newcommand{\adj}{A}
\newcommand{\Hdom}{\Omega}
\DeclareMathOperator{\comW}{Wsne}
\newcommand{\compproj}{L}
\newcommand{\jointtable}{J}
\DeclareMathOperator{\nbhd}{ne}

\DeclareMathOperator{\com}{sne}
\newcommand{\comp}[2]{C(#1,#2)}
\newcommand{\wgt}[2]{\pi(#1,#2)}

\newcommand{\mixcomp}[2]{C}
\newcommand{\mixwgt}[2]{\pi}

\DeclareMathOperator{\prob}{\mathbb{P}}

\newcommand{\oracle}{\mathsf{MixOracle}}

\usepackage[linesnumbered, ruled,vlined]{algorithm2e}

\newcommand{\camadd}[1]{{#1}}

\begin{document}

\maketitle

\begin{abstract}

We study the problem of reconstructing a causal graphical model from data in the presence of latent variables. The main problem of interest is recovering the causal structure over the latent variables while allowing for general, potentially nonlinear dependencies. In many practical problems, the dependence between raw observations (e.g. pixels in an image) is much less relevant than the dependence between certain high-level, latent features (e.g. concepts or objects), and this is the setting of interest. We provide conditions under which both the latent representations and the underlying latent causal model are identifiable by a reduction to a mixture oracle. \camadd{These results highlight an intriguing connection between the well-studied problem of learning the order of a mixture model and the problem of learning the bipartite structure between observables and unobservables.} The proof is constructive, and leads to several algorithms for explicitly reconstructing the full graphical model. We discuss efficient algorithms and provide experiments illustrating the algorithms in practice.
\end{abstract}

\section{Introduction}
\label{sec:intro}

Understanding causal relationships between objects and/or concepts is a core component of human reasoning, and by extension, a core component of artificial intelligence \cite{pearl1988,larranaga2011probabilistic}. 
Causal relationships are robust to perturbations, encode invariances in a system, and enable agents to reason effectively about the effects of their actions in an environment.
Broadly speaking, the problem of inferring causal relationships can be broken down into two main steps: 1) The extraction of high-level causal features from raw data, and 2) The inference of causal relationships between these high-level features. From here, one may consider estimating the magnitude of causal effects, the effect of interventions, reasoning about counterfactuals, etc. 
Our focus in this paper will be the problem of learning causal relationships \emph{between} latent variables, which is closely related to the problem of learning causal representations \citep{scholkopf2021toward}. This problem should be contrasted with the equally important problem of causal inference in the presence of latent confounders \citep[e.g.][]{colombo2012, spirtes2013causal, anandkumar2013, hoyer2006estimation, silva2006learning}; see also Remark~\ref{rem:objective}.

Causal graphical models \cite{pearl2009,pearl1988} provide a natural framework for this problem, and have long been used to model causal systems with hidden variables \cite{richardson2002ancestral,evans2016graphs,evans2014markovian,evans2018margins,evans2019smooth,richardson2017nested}.  
It is well-known that in general, without additional assumptions, a causal graphical model given by a directed acyclic graph (DAG) is not identifiable in the presence of latent variables \cite[e.g.,][]{pearl1988,spirtes2000}.
In fact, this is a generic property of nonparametric structural models: Without assumptions, identifiability is impossible, however, given enough structure, identifiability can be rescued. Examples of this phenomenon include linearity \cite{frot2017robust,chandrasekaran2012latent,anandkumar2013,xie2020generalized,anderson1984estimating}, independence \cite{allman2009,bonhomme2016latent,xie2020generalized}, rank \cite{frot2017robust,chandrasekaran2012latent}, sparsity \cite{anderson1984estimating}, and graphical constraints \cite{anandkumar2013,anandkumar2013learning}.

In this paper, we consider a general setting for this problem with \emph{discrete} latent variables, while allowing otherwise arbitrary (possibly nonlinear) dependencies. The latent causal graph between the latent variables is also allowed to be arbitrary: No assumptions are placed on the structure of this DAG.
We do not assume that the number of hidden variables, their state spaces, or their relationships are known; in fact, we provide explicit conditions under which all of this can be recovered uniquely. 
To accomplish this, we highlight a crucial reduction between the problem of learning a DAG model over these variables---given access only to the observed data---and learning the parameters of a finite mixture model. This observation leads to new identifiability conditions and algorithms for learning causal models with latent structure.

\paragraph{Overview}
Our starting point is a simple reduction of the graphical model recovery problem to three modular subproblems:
\begin{enumerate}
\item The bipartite graph $\Gamma$ between hidden and observed nodes,
\item The latent distribution $\pr(H)$ over the hidden variables $H$, and
\item A directed acyclic graph (DAG) $\Lambda$ over the latent distribution.
\end{enumerate}
From here, the crucial observation is to reduce the recovery problems for $\bipartite$ and $\pr(\lat)$ to the problem of learning a finite mixture over the observed data. The latter is a well-studied problem with many practical algorithms and theoretical guarantees. We do not require parametric assumptions on this mixture, which allows for very general dependencies between the observed and hidden variables. 
From this mixture model, we extract what is needed to learn the full graph structure.

This perspective leads to a systematic, modular approach for learning the latent causal graph via mixture oracles (see Section~\ref{sec:bg} for definitions).
Ultimately, the application of these ideas requires a practical implementation of this mixture oracle, which is discussed in Section~\ref{sec:alg}.

\paragraph{Contributions}
More precisely, we make the following contributions:
\begin{enumerate}
\item (Section~\ref{sec:exact}) We provide general conditions under which the latent causal model $\gr$ is identifiable (Theorem~\ref{thm:main}). Surprisingly, these conditions mostly amount to nondegeneracy conditions on the joint distribution. As we show, without these assumptions identifiability breaks down and reconstruction becomes impossible.
\item (Section~\ref{sec:bipartite}) We carefully analyze the problem of reconstructing $\bipartite$ under progressively weaker assumptions: First, we derive a brute-force algorithm that identifies $\bipartite$ in a general setting  (Theorem~\ref{thm:hidden-bipartite}), and then under a linear independence condition we derive a polynomial-time algorithm based on tensor decomposition and Jennrich’s algorithm  (Theorem~\ref{thm:recovery:eff}).
\item (Section~\ref{sec:Ph}) Building on top of the previous step, where we learn the bipartite graph and sizes of the domains of latent variables, we develop an efficient algorithm for learning the latent distribution $\pr(H)$ from observed data  (Theorem~\ref{thm:Ph}). 
\item (Section~\ref{sec:alg}-\ref{sec:expts}) We implement these algorithms as part of an end-to-end pipeline for learning the full causal graph and illustrate its performance on simulated data.
\end{enumerate}
A prevailing theme throughout is the fact that the hidden variables leave a recognizable ``signature'' in the observed data through the marginal mixture models induced over subsets of observed variables. By cleverly exploiting these signatures, the number of hidden variables, their states, and their relationships can be recovered exactly.

\paragraph{Previous work}
Latent variable graphical models have been extensively studied in the literature; as such we focus only on the most closely related work on causal graphical models here. 
\camadd{Early work on this problem includes seminal work by Martin and VanLehn \cite{martin1995discrete}, Friedman et al.~\cite{friedman1997learning} Elidan et al.~\cite{elidan2000discovering}.}
More recent work has focused on linear models \citep{anandkumar2013,frot2017robust,silva2006learning,xie2020generalized} or known structure \citep{kocaoglu2017causalgan,ding2021gans,shen2020disentangled}. 
When the structure is not known \emph{a priori}, we find ourselves in the realm of \emph{structure learning}, which is our focus.
Less is known regarding structure learning between latent variables for nonlinear models, although there has been recent progress based on nonlinear ICA
\citep{monti2020causal,khemakhem2020variational}. For example,
\citep{yang2020causalvae} proposed CausalVAE, which assumes a linear structural equation model and knowledge of the concept labels for the latent variables, in order to leverage the iVAE model from \citep{khemakhem2020variational}. By contrast, our results make no linearity assumptions and do not require these additional labels. 
\camadd{While this paper was under review, we were made aware of the recent work \cite{markham2020measurement} that studies a similar problem to ours in a general, nonlinear setting under faithfulness assumptions.}
It is also worth noting recent progress on learning discrete Boltzmann machines \citep{bresler2019learning,bresler2020learning}, which can be interpreted as an Ising model with a bipartite structure and a single hidden layer---in particular, there is no hidden causal structure. Nevertheless, this line of work shows that learning Boltzmann machines is computationally hard in a precise sense.
\camadd{More broadly, the problem of learning latent structure has been studied in a variety of other applications including latent Dirichlet allocation \cite{arora2012learning, arora2013practical}, phylogenetics \cite{mossel2005learning,semple2003phylogenetics}, and hidden Markov models \cite{anandkumar2012mixture,gassiat2013finite}.}

A prevailing theme in the causal inference literature has been negative results asserting that 
in the presence of latent variables, causal inference is impossible \cite{grimmer2020ive,robins2003,robins1999impossibility,damour2019multi}. Our results do not contradict this important line of work, and instead adopts a more optimistic tone: We show that under reasonable assumptions---essentially that the latent variables are discrete and well-separated---identifiability and exact recovery of latent causal relationships is indeed possible. 
This optimistic approach is implicit in recent progress on visual relationship detection \citep{newell2017pixels}, causal feature learning \citep{chalupka2017causal,lopez2017discovering}, and interaction modeling \citep{li2020causal,kipf2018neural}. In this spirit, our work provides theoretical grounding for some of these ideas.

\paragraph{Mixture models and clustering}
While our theoretical results in Sections~\ref{sec:exact}-\ref{sec:Ph} assume access to a mixture oracle (see Definition~\ref{defn:mixoracle}), in Section~\ref{sec:alg} we discuss how this oracle can be implemented in practice. To provide context for these results, we briefly mention related work on learning mixture models from data. Mixture models can be learned under a variety of parametric and nonparametric assumptions. Although much is known about parametric models \citep[e.g.][]{lindsay1995}, of more interest to us are nonparametric models in which the mixture components are allowed to be flexible, such as mixtures of product distributions \citep{hall2003,gordon2021hadamard}, grouped observations \citep{ritchie2020consistent,vandermeulen2016mixture} and general nonparametric mixtures \citep{aragam2018npmix,shi2009}. In each of these cases, a mixture oracle can be implemented without parametric assumptions. In practice, we use clustering algorithms such as $K$-means or hierarchical clustering to implement this oracle. We note also that the specific problem of consistently estimating the order of a mixture model, which will be of particular importance in the sequel, has been the subject of intense scrutiny in the statistics literature \citep[e.g.][]{manole2020estimating,koltchinskii2000empirical,dacunha1997estimation,chen2009order}.

\paragraph{Broader impacts and societal impact}
Latent variable models have numerous practical applications. Many of these applications positively address important social problems, however, these models can certainly be applied nefariously. For example, if the latent variables represent private, protected information, our results imply that this hidden private data can be leaked into publicly released data, which is obviously undesirable. Understanding how to infer unprotected data while safeguarding protected data is an important problem, and our results shed light on when this is and isn't possible.

\paragraph{Notation}
We say that a distribution $\pr(\ver)$ satisfies the \emph{Markov property} with respect to a DAG $\gr=(\ver,E)$ if
\begin{equation}
    \prob(V) = \prod_{v\in V} \prob(v\mid \pa_{G}(v)).
\end{equation}
An important consequence of the Markov property is that it allows one to read off conditional independence relations from the graph $\gr$. More specifically, we have the following \citep[see][for details]{pearl1988,spirtes2000}: 
\begin{itemize}
    \item For each $v\in\ver$, $v$ is independent of its non-descendants, given its parents. 
    \item For disjoint subsets $\ver_1,\ver_2,\ver_3\subset\ver$, if $\ver_1$ and $\ver_2$ are $d$-separated given $\ver_3$ in $\gr$, then $\ver_1\indep\ver_2\given\ver_3$ in $\pr(\ver)$.
\end{itemize}
The concept of $d$-separation (see \S3.3.1 in \citep{pearl1988} or \S2.3.4 in \citep{spirtes2000}) gives rise to a set of independence relations, often denoted by $\mathcal{I}(\gr)$. The Markov property thus implies that $\mathcal{I}(\gr)\subset\mathcal{I}(\ver)$, where $\mathcal{I}(\ver)$ is the collection of all valid conditional independence relations over $\ver$. When the reverse inclusion holds, we say that $\pr(\ver)$ is \emph{faithful} to $\gr$ (also that $\gr$ is a \emph{perfect map} of $\ver$). Although the concepts of faithfulness and $d$-separation will not be needed in the sequel, we have included this short discussion for completeness and context (cf. Section~\ref{sec:exact}).

Throughout this paper, we use standard notation such as $\pa(j)$ for parents, $\ch(j)$ for children, and $\nbhd(j)$ for neighbors.
Specifically, we define
\begin{itemize}
    \item The parents of a node $v\in\ver$ are denoted by $\pa(v)=\{u\in\ver: (u,v)\in E\}$;
    \item The children of a node $v\in\ver$ are denoted by $\ch(v)=\{u\in\ver: (v,u)\in E\}$;
    \item The neighborhood of a node $v\in\ver$ is denoted by $\nbhd(v)=\pa(v)\cup\ch(v)$.
\end{itemize}
Given a subset $\ver'\subset\ver$, $\pa(\ver'):=\cup_{j\in\ver'}\pa(j)$ and given a subgraph $\gr'\subset \gr$, $\pa_{\gr'}(\ver'):=\pa(\ver')\cap \gr'$, with similar notation for children and neighbors. We let $\adj\in\{0,1\}^{|\obs|\times |\lat|}$ denote the adjacency matrix of $\bipartite$ and denote its columns by $a_{j}\in\{0,1\}^{|\obs|}$.
Finally, we adopt the convention that $\lat$ is identified with the indices $[\nlat]=\{1,\ldots,\nlat\}$, and similar $\obs$ is identified with $[\nobs]=\{1,\ldots,\nobs\}$. In particular, we use $\pa(i)$ and $\pa(H_i)$ interchangeably when the context is clear.

\section{Background}
\label{sec:bg}

\begin{figure}[t]
\centering
\begin{subfigure}[b]{0.2\textwidth}
\begin{tikzpicture}[node distance={15mm}, 
                    latent/.style = {draw, circle, orange},
                    obs/.style = {draw, circle, blue}
                    ] 
\node[latent] (Z1) {$H_1$}; 
\node[obs] (X1) [below left of=Z1] {$X_1$}; 
\node[obs] (X2) [below right of=Z1] {$X_2$};

\draw[->, blue, dashed] (Z1) -- (X1);
\draw[->, blue, dashed] (Z1) -- (X2);
\end{tikzpicture} 
\caption{A single hidden state.}
\end{subfigure}
~
\begin{subfigure}[b]{0.3\textwidth}
\centering
\begin{tikzpicture}[node distance={15mm}, 
                    latent/.style = {draw, circle, orange},
                    obs/.style = {draw, circle, blue}
                    ] 
\node[latent] (Z1) {$H_1$}; 
\node[latent] (Z2) [right of=Z1] {$H_2$};
\node[obs] (X1) [below of=Z1] {$X_1$}; 
\node[obs] (X2) [below of=Z2] {$X_2$};


\draw[->, blue, dashed] (Z1) -- (X1);
\draw[->, blue, dashed] (Z2) -- (X2);
\end{tikzpicture} 
\caption{Two independent hidden states.}
\end{subfigure}
~
\begin{subfigure}[b]{0.2\textwidth}
\centering
\begin{tikzpicture}[node distance={15mm}, 
                    latent/.style = {draw, circle, orange},
                    obs/.style = {draw, circle, blue}
                    ] 
\node[latent] (Z1) {$H_1$}; 
\node[latent] (Z2) [right of=Z1] {$H_2$};
\node[latent] (Z3) [right of=Z2] {$H_3$};
\node[obs] (X1) [below of=Z1] {$X_1$}; 
\node[obs] (X2) [below of=Z2] {$X_2$};
\node[obs] (X3) [below of=Z3] {$X_3$};

\draw[->, orange] (Z1) -- (Z2);
\draw[->, orange] (Z3) -- (Z2);
\draw[->, orange] (Z3) to [out=140,in=40] (Z1);

\draw[->, blue, dashed] (Z1) -- (X1);
\draw[->, blue, dashed] (Z2) -- (X2);
\draw[->, blue, dashed] (Z2) -- (X3);
\draw[->, blue, dashed] (Z3) -- (X3);
\end{tikzpicture} 
\caption{Three dependent hidden states.}
\end{subfigure}
\caption{Illustration of the basic model. Note that there are no edges between observed variables or edges oriented from observed to hidden. (a) A latent variable model with a single hidden state; i.e. a mixture model. (b)-(c) Two examples of latent variable models with more complicated hidden structure. }
\label{fig:model}
\end{figure}
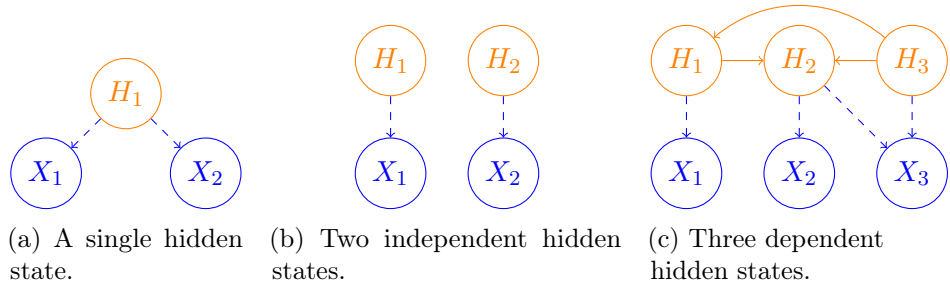

Let $\gr=(\ver,\edg)$ be a DAG with $\ver=(\obs,\lat)$, where $\obs\in\R^{\nobs}$ denotes the observed part and $\lat\in\Hdom:=\Hdom_{1}\times\cdots\times\Hdom_{\nlat}$ denotes the hidden, or latent, part. Throughout this paper, we assume that each $\Hdom_{i}$ is a discrete space with $|\Hdom_{i}|\ge 2$. 
We assume further that there are no edges between observed variables and no edges from observed to hidden variables, and that the distribution of $\ver$ satisfies the Markov property with respect to $\gr$ (see the supplement for definitions).
Under these assumptions, $\gr$ decomposes as the union of two subgraphs $\gr=\bipartite\cup\latentgr$, where $\bipartite$ is a directed, bipartite graph of edges pointing from $\lat$ to $\obs$, and $\latentgr$ is a DAG over the latent variables $\lat$.
Similar assumptions have appeared in previous work \citep{anandkumar2013,xie2020generalized,markham2020measurement}, and although nontrivial, they encapsulate our keen interest in reconstructing the structure $\latentgr$ amongst the latent variables, and captures relevant applications where the relationships between raw observations is less relevant than so-called ``causal features'' \citep{chalupka2014visual,chalupka2017causal}.
See Figure~\ref{fig:model} for examples.

Throughout this paper, we use standard notation such as $\pa(j)$ for parents, $\ch(j)$ for children, and $\nbhd(j)$ for neighbors. Given a subset $\ver'\subset\ver$, $\pa(\ver'):=\cup_{j\in\ver'}\pa(j)$ and given a subgraph $\gr'\subset \gr$, $\pa_{\gr'}(\ver'):=\pa(\ver')\cap \gr'$, with similar notation for children and neighbors. We let $\adj\in\{0,1\}^{|\obs|\times |\lat|}$ denote the adjacency matrix of $\bipartite$ and denote its columns by $a_{j}\in\{0,1\}^{|\obs|}$.

\begin{remark}
\label{rem:objective}
\camadd{Our goal is to learn the hidden variables $\lat$ and the causal graph between them, defined above by $\latentgr$. To accomplish this, our main result (Theorem~\ref{thm:main}) shows how to identify $(\bipartite,\pr(\lat))$, from which $\latentgr$ can be recovered (see Section~\ref{sec:exact} for details). It is important to contrast this problem with problems involving latent confounders \citep[e.g.][]{colombo2012,spirtes2013causal,anandkumar2013,hoyer2006estimation,silva2006learning}, where the goal is to learn the causal graph between the observed variables $\obs$. In our setting, there are no edges between the observed variables.}
\end{remark}

\subsection{Assumptions}
It is well-known that without additional assumptions, the latent variables $\lat$ cannot be identified from $\obs$, let alone the DAG $\latentgr$.
For example, we can always replace a pair of distinct hidden variables $\lat_{i}$ and $\lat_{j}$ with a single hidden variable $\lat_{0}$ that takes values in $\Hdom_{i}\times \Hdom_{j}$. Similarly, a single latent variable can be split into two or more latent variables. In order to avoid this type of degeneracy, we make the following  assumptions:

\begin{assumption}[No twins]
\label{assm:twins}
For any hidden variables $\lat_{i}\neq \lat_{j}$ we have $\nbhd_{\Gamma}(H_i)\neq\nbhd_{\Gamma}(H_j)$.
\end{assumption}

\begin{assumption}[{\camadd{Maximality}}]
\label{assm:minimal}
There is no DAG $\gr' = ((X, H'), E')$ such that: 
\begin{enumerate}
    \item $\pr(X, H')$ is Markov with respect to $\gr'$;
    \item \camadd{$\gr'$ is obtained from $\gr$ by splitting a hidden variable (equivalently, $\gr$ is obtained from $\gr'$ by merging a pair of vertices)};
    \item $G'$ satisfies Assumption~\ref{assm:twins}.
\end{enumerate}   
\end{assumption}

These assumptions are necessary for the recovery of $\latentgr$ in the sense that, without these assumptions, latent variables can be created or destroyed without changing the observed distribution $\pr(\obs)$. \camadd{Informally, the maximality assumption says that if there are several DAGs that are Markov with respect to the given distribution, we are interested in recovering the most informative among them.} 
Finally, we make a mild assumption on the probabilities, in order to avoid degenerate cases where certain configurations of the latent variables have zero probability:
\begin{assumption}[Nondegeneracy]
\label{assump:strong}
The distribution over $\ver=(\obs,\lat)$ satisfies:
\begin{enumerate}[label=(\alph*)] 
\item\label{assump:strong:NZ} $\prob(H = h)>0$ for all $h\in \Hdom_1 \times \ldots \times \Hdom_k$.
\item\label{assump:strong:SDC} For all $\nhbd\subset\obs$ and $a\ne b$, $\prob(\nhbd|\pa(\nhbd) = a)\ne\prob(\nhbd|\pa(\nhbd) = b)$, where $a$ and $b$ are distinct configurations of $\pa(\nhbd)$.
\end{enumerate}
\end{assumption}
Without this nondegeneracy condition, $\lat$ cannot be identified; see Appendix~\ref{app:example-nonidentifiable} for details.

\subsection{Mixture oracles}
Let $\nhbd\subset\obs$ be a subset of the observed variables. We can always write the marginal distribution $\pr(\nhbd)$ as
\begin{align}
\label{eq:hidden:mix}
\pr(\nhbd)
= \sum_{h\in\Hdom}\pr(\lat=h)\pr(\nhbd\given\lat=h).
\end{align}
When $\nhbd=\obs$, this can be interpreted as a mixture model with $\nmix:=|\Hdom|$ components. When $\nhbd\subsetneq\obs$, however, multiple components can ``collapse'' onto the same component, resulting in a mixture with fewer than $\nmix$ components. Let $\ncomp(\nhbd)$ denote this number, so that we may define a discrete random variable $\mix$ with $\ncomp(\nhbd)$ states such that for all $j\in[\ncomp(\nhbd)]$, we have 
\begin{align}
\pr(\nhbd)
= \sum_{j=1}^{\ncomp(\nhbd)}\underbrace{\pr(\mix=j)}_{:=\wgt{\nhbd}{j}}\underbrace{\pr(\nhbd\given\mix=j)}_{:=\comp{\nhbd}{j}}
= \sum_{j=1}^{\ncomp(\nhbd)} \wgt{\nhbd}{j}\comp{\nhbd}{j}.
\end{align}
Then $\wgt{\nhbd}{j}$ is the weight of the $j$th mixture component over $\nhbd$, and $\comp{\nhbd}{j}$ is the corresponding $j$th component.
It turns out that these probabilities precisely encode the conditional independence structure of $\lat$. To make this formal, we define the following oracle:
\begin{definition}
\label{defn:mixoracle}
A \emph{mixture oracle} is an oracle that takes $\nhbd\subset\obs$ as input and returns the number of components $\ncomp(\nhbd)$ as well as the weights $\wgt{\nhbd}{j}$ and components $\comp{\nhbd}{j}$ for each $j\in[\ncomp(\nhbd)]$. This oracle will be denoted by $\oracle(\nhbd)$.
\end{definition}

Although our theoretical results are couched in the language of this oracle, we provide practical implementation details in Section~\ref{sec:alg} and experiments to validate our approach in Section~\ref{sec:expts}.

A sufficient condition for the existence of a mixture oracle is that the mixture model over $\obs$ is identifiable. This is because identifiability implies that the number of components $\nmix$, the weights $\pr(\mix=j)$, and the mixture components $\pr(\obs\given\mix=j)$ are determined by $\pr(\obs)$. The marginal weights $\wgt{\nhbd}{j}$ and components $\comp{\nhbd}{j}$ can then be recovered by simply projecting the full mixture over $\obs$ onto $\nhbd$.

\begin{remark}
In fact, we do not need the full power of $\oracle$. For our algorithms it is sufficient to have access to $\ncomp(\nhbd)$ for a sufficiently large family of $S\subset X$, the list of weights $\wgt{X}{j}$, and a map that relates components in the full mixture over $X$ to the components in the marginal mixtures over each variable $X_i$ (see Section~\ref{sec:Ph} for details).
\end{remark}

Before concluding this section, we note an important consequence of Assumption~\ref{assump:strong} that will be used in the sequel:

\begin{observation}\label{obs:counting} Under Assumption~\ref{assump:strong}, for any $\nhbd\subseteq\obs$ 
\[k(S) = \prod\limits_{\lat_{i}\in \pa(\nhbd)} \dim(\lat_{i}) =: \dim(\pa(\nhbd)) .\]
\end{observation}
\begin{proof}
By the Markov property, $S$ is independent of $H\setminus \pa(S)$. There are $\dim(\pa(S))$ possible assignments to the hidden variables in $\pa(S)$ and by Assumption~\ref{assump:strong}, distinct assignments to the hidden variables induce distinct components in the marginal distribution $P(S)$. Hence, by definition, $k(S) = \dim(\pa(S))$.
\end{proof}

\section{Recovery of the latent causal graph}
\label{sec:exact}

We first consider the oracle setting in which we have access to $\oracle(S)$.

Observe that the problem of learning $\gr$ can be reduced to learning $(\bipartite,\pr(\lat))$:
Since we can decompose $\gr$ into a bipartite subgraph $\bipartite$ and a latent subgraph $\latentgr$, it suffices to learn these two components separately. We then further reduce the problem of learning $\latentgr$ to learning the latent distribution $\pr(\lat)$.
First, we will show how to reconstruct $\bipartite$ from $\oracle(S)$. Then, we will show how to learn the latent distribution $\pr(\lat)$ from $\oracle(S)$. 

Thus, the problem of learning $\gr$ is reduced to the mixture oracle:
\begin{align*}
    \gr 
    \to (\bipartite,\pr(\lat))
    \to \oracle(S).
\end{align*}

In the sequel, we focus our attention on recovering $(\bipartite,\pr(\lat))$. In order to recover $\pr(\lat)$, we will require the following assumption:
\begin{assumption}[Subset condition]\label{assum:ssc}
We say that the bipartite graph $\bipartite$ satisfies the subset condition (SSC) if for any pair of distinct hidden variables $\lat_i, \lat_j$ the set $\nbhd_{\Gamma}(H_i)$ is not a subset of $\nbhd_{\Gamma}(H_j)$. 
\end{assumption}

This assumption is weaker than the common ``anchor words" assumption from the topic modeling literature. The latter assumption says that every topic has a word that is unique to this topic, and it is commonly assumed for efficient recovery of latent structure~\cite{arora2012learning, arora2013practical}.

Under Assumption~\ref{assum:ssc}, we have the following key result:
\begin{theorem}
\label{thm:main}
Under Assumptions~\ref{assm:twins},~\ref{assm:minimal},~\ref{assump:strong}, and~\ref{assum:ssc}, $(\bipartite,\pr(\lat))$ can be reconstructed from $\pr(X)$ and $\oracle(S)$. Furthermore, if additionally the columns of the bipartite adjacency matrix $\adj$ are linearly independent, there is an efficient algorithm for this reconstruction.
\end{theorem}
The proof is constructive and leads to an efficient algorithm as alluded to in the previous theorem. An overview of the main ideas behind the proof of this result are presented in Sections~\ref{sec:bipartite} and~\ref{sec:Ph}; the complete proof of this theorem can be found in Appendices~\ref{app:bipartite}-\ref{app:proof:main}.

\camadd{As presented, Theorem~\ref{thm:main} leaves two aspects of the problem unresolved: 1) Under what conditions does $\oracle(S)$ exist, and 2) How can we identify $\latentgr$ from $\pr(\lat)$? As it turns out, each of these problems is well-studied in previous work, which explains our presentation of Theorem~\ref{thm:main}. For completeness, we address these problems briefly below.}

\camadd{
\paragraph{Existence of $\oracle(S)$}

A mixture oracle exists if the mixture model over $X$ is identifiable.
As discussed in Section~\ref{sec:intro}, such identifiability results are readily available in the literature. For example, assume that for every $S\subseteq X$, the mixture model \eqref{eq:hidden:mix} comes from any of the following families:
\begin{enumerate}
    \item a mixture of gaussian distributions~\cite{teicher1963identifiability, yakowitz1968identifiability}, or
    \item a mixture of Gamma distributions~\cite{teicher1963identifiability}, or
    \item an exponential family mixture~\cite{yakowitz1968identifiability}, or
    \item a mixture of product distributions~\cite{teicher1967identifiability}, or
    \item a well-separated (i.e. in TV distance) nonparametric mixture~\cite{aragam2018npmix}.
\end{enumerate}
Then $(\bipartite,\pr(\lat))$ is identifiable.
The list above is by no means exhaustive, and many other results on identifiability of mixture models are known (e.g., see the survey~\cite{mclachlan2019finite}). 
}

\paragraph{Identifiability of $\latentgr$}
Once we know $\pr(\lat)$ (e.g. via  Theorem~\ref{thm:main}), identifying $\latentgr$ from $\pr(\lat)$ is a well-studied problem with many solutions \citep{spirtes2000,pearl1988}. For simplicity, it suffices to assume that $\pr(\lat)$ is faithful to $\latentgr$, which implies that $\latentgr$ can be learned up to Markov equivalence. This assumption is \emph{not} necessary, and any number of alternative identifiability assumptions on $\pr(\lat)$ can be plugged in place of faithfulness, for example triangle faithfulness \citep{spirtes2014uniformly}, independent noise \citep{shimizu2006,peters2014}, post-nonlinearity \citep{zhang2009}, equality of variances \citep{peters2013,gao2020npvar}, etc.

\section{Learning the bipartite graph}
\label{sec:bipartite}

In this section we outline the main ideas behind the recovery of $\bipartite$ in Theorem~\ref{thm:main}.
We begin by establishing conditions that ensure $\bipartite$ is identifiable, and then proceed to consider efficient algorithms for its recovery.

\subsection{Identifiability result}

We study a slightly more general setup in which the identifiability of $\bipartite$ depends on how much information we request from the $\oracle$. Clearly, we want to rely on $\oracle$ as little as possible. As the proofs in the supplement indicate, the only information required for this step are the number of components. Neither the weights nor the components are needed.

\begin{definition}\label{def:k-recoveroble}
We say that $\Gamma$ is $t$-recoverable if $\Gamma$ can be uniquely recovered from $X$ and the sequence $(\oracle(S)\mid |S|\le t)$.
\end{definition}

\begin{theorem}\label{thm:hidden-bipartite}
Let $\Gamma$ be the bipartite graph between $X$ and $H$. 
\begin{enumerate}[label=(\alph*)]
    \item\label{thm:hidden-bipartite:exp} Assume that $\nbhd_{\Gamma}(H_i)\neq \nbhd_{\Gamma}(H_j)$ for any $i\neq j$. Then $\Gamma$  and $\dim(H_i)$ are $n$-recoverable.
    \item\label{thm:hidden-bipartite:eff} Let $t\geq 3$. Assume that for every $S\subseteq H$ with $|S|\geq 2$ we have
\[ \dim \Span \{a_j \mid j\in S\}\geq \dfrac{2}{t}|S|+1,\]
then $\Gamma$ and $\dim(H_i)$ are $t$-recoverable.
\end{enumerate}
  
\end{theorem}

Note that Assumption~\ref{assum:ssc} implies the assumption in Theorem~\ref{thm:hidden-bipartite}\ref{thm:hidden-bipartite:exp}. Finally, as in Section~\ref{sec:bg}, we argue that in the absence of additional assumptions, this assumption is in fact necessary:
\begin{observation}\label{obs:twins-gamma-norecovery}
If there is a pair of distinct variables $H_i, H_j\in H$ such that $\nbhd_{\Gamma}(H_1) = \nbhd_{\Gamma}(H_2)$, then $\Gamma$ is not $n$-recoverable.  
\end{observation}

\subsection{Ideas behind the recovery}\label{sec:recovery-ideas}

In Corollary~\ref{cor:weight-observed-hidden} below, we recast Observation~\ref{obs:counting} as an additive identity. This transforms the problem of learning $\Gamma$ into an instance of more general problem that is discussed in the appendix. The results of this section apply to this more general version.

\begin{corollary}\label{cor:weight-observed-hidden}
Assume that Assumptions~\ref{assump:strong} hold. For $H_i \in H$ define $w(H_i) = \log(\dim(H_i))$. Then for every set $S\subseteq X$
\begin{equation}
    \log(k(S)) = \sum\limits_{H_i\in \pa(S)} w(H_i).
\end{equation}
\end{corollary}

In order to argue about the causal structure of the hidden variables we first need to identify the variables themselves. By Assumption~\ref{assm:twins}, every hidden variable leaves a ``signature'' among the observed variables, which is the set $\nbhd_{\Gamma}(H_i)$ of observed variables it affects. In particular, note that $H_i\in \bigcap_{X_s\in \nbhd_{\Gamma}(H_i)} \pa(X_s)$, and if there is no $H_j$ with $\nbhd_{\Gamma}(H_i)\subset \nbhd_{\Gamma}(H_j)$, then $H_i$ is the unique element of the intersection. The lemma above allows us to extract information about the union of parent sets, and we wish to turn it into the information about intersections. This motivates the following definitions.
\begin{definition} Let $\Gamma$ and $w$ be as above. Define
\begin{equation}
\label{eq:sne:Wsne}
    \com_{\Gamma}(S) = \bigcap\limits_{x\in S} \nbhd_{\Gamma}(x) \quad \text{and}\quad  \comW_{\Gamma}(S) = \sum\limits_{v\in \com_{\Gamma}(S)} w(v)  
\end{equation}
\end{definition}

\begin{lemma}\label{lem:com-weight}
For a set $S\subseteq X$ we have
\begin{equation}\label{eq:common-par-comp-gr}
    \comW_{\Gamma}(S) = \sum\limits_{U\subseteq S, U\neq \emptyset} (-1)^{|U|+1}W_{\Gamma}(U),\quad  \text{where}\quad W_{\Gamma}(S) = \sum\limits_{v \in \nbhd_{\Gamma}(S)} w(v).
\end{equation}
\end{lemma}
The proof of this lemma is a simple application of the Inclusion-Exclusion principle.

\begin{remark}\label{rem:compute:scores}
The RHS of Eq.~\eqref{eq:common-par-comp-gr} only depends on $W$ evaluated on subsets of $S$. Thus, in particular, if $|S|\leq t$ to compute $\comW(S)$ it is enough to know $\oracle$ on all sets of size $\leq t$.
\end{remark}

Finally, the values of the function $\comW_{\Gamma}$ can be organized into a tensor, and from here the problem of learning $\Gamma$ can be cast as decomposition problem for this tensor. These proof details are spelled out in Appendix~\ref{app:bipartite}; in the next section we illustrate this procedure for the special case of 3-recovery.

\subsection{Efficient $3$-recovery}

Under a simple additional assumption $\Gamma$ can be recovered efficiently. We are primarily interested in the case $t = 3$. The main idea is to note that a rank-three tensor involving the columns of $\adj$ can be written in terms of $\comW_{\Gamma}$. We can then apply Jennrich's algorithm \citep{harshman} 
to decompose the tensor and recover these columns, which yield $\Gamma$.  
To see this, let $I = (i_1, i_2, i_3)\subseteq X$ be a triple of indices, and note that
\begin{equation}\label{eq:3tensor-formula}
    \sum\limits_{j\in H} w(j)(a_j)_{i_1}(a_j)_{i_2}(a_j)_{i_3} = \Big(\sum\limits_{j\in H} w(j) a_j\otimes a_j \otimes a_j\Big)_{(i_1, i_2, i_3)} = \comW_{\Gamma}(I).
\end{equation}

\begin{theorem}
\label{thm:recovery:eff}
Assume that the columns of $\adj$ are linearly independent. Then $\Gamma$ and $\dim(H_i)$, for all $i$, are $3$-recoverable in $O(n^3)$ space and $O(n^4)$ time.
\end{theorem}
\begin{proof}
It takes $O(n^3)$ space and $O(n^3)$ time to compute $M_3$ and then Jennrich's algorithm can decompose the tensor in $O(n^3)$ space and $O(n^4)$ time.
\end{proof}

\section{Learning the latent distribution}
\label{sec:Ph}

In this section we outline the main ideas behind the recovery of $\pr(\lat)$ in Theorem~\ref{thm:main}. 

\begin{remark}
Since the variables $H$ are not observed, $\oracle(S)$ only tells us the set \[\{(i, \wgt{S}{i}, \comp{S}{i})\mid i \in [k(S)] \}. \]
But the correspondence $\Hdom \ni h \leftrightarrow j\in [K]$ between a possible tuple $h$ of values of hidden variables and the corresponding mixture component is unknown. 
\end{remark}

Since the values of $H$ are not observed, we may learn this correspondence only up to a relabeling of $\Hdom_{i}$. By definition, the input distribution has $K= |\Omega|$ mixture components over $X$ and $k_i=\ncomp(\obs_{i})$ mixture components over $X_i$. Fix any enumeration of these components by $[K]$ and $[k_i]$, respectively. 
To recover the correspondence $\Hdom \ni h \leftrightarrow j\in [K]$, we will need access to the map 
 \begin{equation}
      \compproj:[K]\to[k_{1}]\times\cdots\times[k_{n}],
 \end{equation} 
 defined so that $[L(j)]_i$ equals to the index of the mixture component $\comp{X}{j}$ (marginalized over $X_i$) in the marginal distribution over $X_i$.
 Crucially, this discussion establishes that $\compproj$ can be computed from a combination of $\oracle(X)$ and $\oracle(X_{i})$ for each $i$.
 
 The map $\compproj$ encodes partial information about the causal structure in $G$. Indeed, if $h_1, h_2\in \Hdom$ are a pair of states of hidden variables $H$ that coincide on $\pa(X_i)$ for some $X_i \in X$, then by the Markov property the components that correspond to $h_1$ and $h_2$ should have the same marginal distribution over $X_i$. 

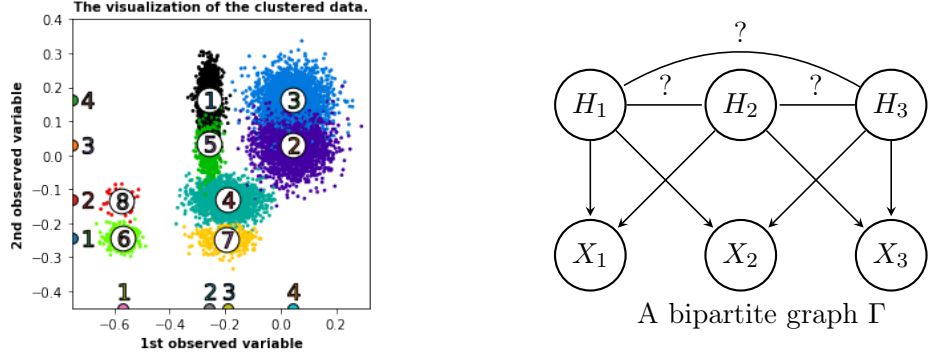
\begin{figure}
\begin{subfigure}{0.6\textwidth}
    \centering
    \includegraphics[scale = 0.5]{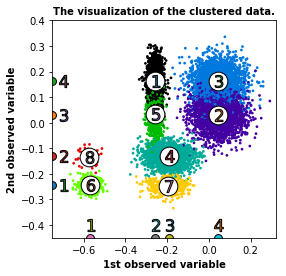}
    \label{fig:my_label}
\end{subfigure}
\begin{subfigure}{0.4\textwidth}
\begin{tikzpicture}[
            > = stealth, 
            shorten > = 1pt, 
            auto,
            node distance = 2cm, 
            semithick 
        ]

        \tikzstyle{every state}=[
            draw = black,
            thick,
            fill = white,
            minimum size = 4mm
        ]
        \node[state] (h1) {$H_1$};
        \node[state] (x1) [below of=h1] {$X_1$};
        \node[state] (h2) [right of=h1] {$H_2$};
        \node[state] (h3) [right of=h2] {$H_3$};
        \node[state] (x2) [below  of=h2] {$X_2$};
        \node[state] (x3) [below  of=h3] {$X_3$};

        \path[->] (h1) edge (x1);
        \path[->] (h1) edge (x2);
        \path[->] (h2) edge (x1);
        \path[->] (h2) edge (x3);
        \path[->] (h3) edge (x2);
        \path[->] (h3) edge (x3);
        \path (h3) edge node[above]{?} (h2);
        \path (h3) edge [bend right=30] node[above]{?} (h1);
        \path (h1) edge node[above]{?} (h2);
        
         \node [below right=.5cm, align=left,text width=8cm] at (x1)
        {
           A bipartite graph $\Gamma$
        };
\end{tikzpicture}
\end{subfigure}
\caption{Example of a latent DAG and corresponding mixture distribution}\label{fig:example:PHlearning}
\end{figure}

\begin{example}
\label{ex:Ph:reconstruct}
Consider the DAG on Figure~\ref{fig:example:PHlearning}. We do not make any assumptions about the causal structure between hidden variables. This DAG has $3$ hidden variables, and we assume that each of them takes values in the set $\{0, 1\}$. Then by Assumption~\ref{assump:strong}, every observed variable is a mixture of $4$ components, while the distribution on $X$ is a mixture of $8$ components. Note that the anchor word assumption is violated here, while (SSC) assumption is satisfied. The map $L:[8]\rightarrow [4]\times [4]\times [4]$ for an example as in Fig.~\ref{fig:example:PHlearning} has form
\begin{equation*}
\begin{matrix}
    i: &1 & 2 & 3 & 4 & 5 & 6 & 7 & 8\\
    L(i): & (2, 4, 3), &  (4, 3, 4), &  (4, 4, 2), &  (3, 2, 4), &  (2, 3, 1), &  (1, 1, 3), &  (3, 1, 2), &  (1, 2, 1)
    \end{matrix}
\end{equation*}
Our goal is to find the correspondence between $h\in \Hdom = \{0, 1\}^3$ and $i\in [8]$. (The projection on the third variable is not shown on Figure~\ref{fig:example:PHlearning}, so the third coordinate of $L$ cannot be deduced from the plot.)  
\end{example}

We now show that there is an algorithm that exactly recovers $\prob(H)$ from the bipartite graph $\Gamma$, the map $\compproj:[K]\to[k_{1}]\times\cdots\times[k_{n}]$, and the mixture weights (probabilities) $\{\wgt{X}{i} \mid i\in [K] \} = \{\prob(Z=i)\mid i\in [K]\}$.
Each of these inputs can be computed from $\oracle$. 

\begin{definition}
Let $\jointtable$ be an order-$m$ tensor whose $i$-th mode is indexed by values of $H_i$, such that $\jointtable(h_1, h_2, \ldots, h_m) =  \prob(H = h)$. That is, $\jointtable$ is the joint probability table of $H$.
\end{definition}

\begin{theorem}
\label{thm:Ph}
Suppose Assumptions~\ref{assump:strong} and~\ref{assum:ssc} hold. Then the correspondence $\Hdom \ni h\leftrightarrow \comp{X}{i}$ and the tensor $\jointtable(h_1, h_2, \ldots, h_m) = \prob(H = (h_1, h_2, \ldots, h_m))$ can be efficiently reconstructed from $L$, $\Gamma$ and $\{\wgt{X}{i}\}_{i\in [K]}$. 
\end{theorem}

\begin{remark}\label{obs:ssc-assum-violated} 
If Assumption~\ref{assum:ssc} is violated, then in general $\jointtable$ cannot be reconstructed uniquely and moreover, $G$ cannot be uniquely identified. See Appendix~\ref{app:proof:Ph} for details.
\end{remark}

\subsection{Examples of Algorithm~\ref{algo: ph learning}}

To illustrate this algorithm, in this section we illustrate how it works on Example~\ref{ex:Ph:reconstruct}. The basic idea is the following: We start by arbitrarily assigning a component $\comp{\obs}{i}$---and hence its corresponding probability $\wgt{\obs}{i}$ to some hidden state $h^*=(h_1,\ldots,h_m)$. This assignment amounts to declaring $\pr(H_1=h_1,\ldots,H_m=h_m)=\wgt{\obs}{i}$ and $\pr(\obs\given H_1=h_1,\ldots,H_m=h_m)=\comp{\obs}{i}$.
The choice of initial state $h^*$ here is immaterial; this can be done without loss of generality since the values of the hidden variables can be relabeled without changing anything. From here we proceed inductively by considering hidden states that differ from the previously identified states by in exactly one coordinate. In the example below, we start with $h^*=(0,\ldots,0)$ and then use this as a base case to identify $h^*+e_i$ for each $i=1,\ldots,\nlat$, where
\begin{align*}
    (e_i)_j
    =\begin{cases}
    1 & i=j \\
    0 & i\ne j.
    \end{cases}
\end{align*}

Note that $h^*$ and $e_i$ differ in exactly one coordinate. We then repeat this process until all states have been exhausted. The following example illustrates the procedure and explains how Lemma~\ref{lem:hidden-directions} helps to resolve the ambiguity regarding the assignment of components to hidden states in each step.

\begin{example} Consider the DAG $\gr$ in Fig.~\ref{fig:example:PHlearning}. 
It has $3$ hidden variables, each of which takes values in $\{0, 1\}$. By Assumption~\ref{assump:strong} every observed variable is a mixture of $4$ components, while the distribution on $X$ is a mixture of $8$ components. Note that the anchor word assumption is violated here, while SSC (Assumption~\ref{assum:ssc}) is satisfied. The map $L:[8]\rightarrow [4]\times [4]\times [4]$ 
can be written as:
\begin{equation*}
\begin{matrix}
    i: &1 & 2 & 3 & 4 & 5 & 6 & 7 & 8\\
    L(i): & (2, 4, 3), &  (4, 3, 4), &  (4, 4, 2), &  (3, 2, 4), &  (2, 3, 1), &  (1, 1, 3), &  (3, 1, 2), &  (1, 2, 1)
    \end{matrix}
\end{equation*}
We want to find the correspondence between $h\in \Hdom = \{0, 1\}^3$ and $i\in [8]$. 

We start by picking an arbitrary component, say 1, and assign it to $(H_1, H_2, H_3) = (0, 0, 0)$. 
Next, we make use of Lemma~\ref{lem:hidden-directions}. Since we know $\bipartite$, we know $\ch(H_i)$ for each $i$. In particular, for the hidden variable $H_1$, we know $\ch(H_1) = \{X_1, X_2\}$. This implies that if $H_2, H_3$ are fixed while $H_1$ changes its value, then the component of $X_3$ is unchanged. It follows that the third coordinate of $L$ is also unchanged.
This gives us a way to pair up the components that have the same third coordinate $L(i)_3$; the pairs are $(1, 6)$, $(2, 4)$, $(3, 7)$ and $(5, 8)$. By our previous observation, these pairs are in one-to-one correspondence with unique states of $(H_1,H_2)=(h_1,h_2)$, and each pair identifies the pair of components $(P(X\given H_1=0,H_2=h_2,H_3=h_3), P(X\given H_1=1,H_2=h_2,H_3=h_3))$. Note that at this stage, there is still ambiguity as to which coordinate of each pair corresponds to which component.

Similarly, we can pair up the components that correspond to assignments of hidden variables that differ only in the value of $H_2$. The pairs are $(1, 3)$, $(2, 5)$, $(4, 8)$ and $(6, 7)$. Finally, for $H_3$ the pairs are $(1, 5)$, $(2, 3)$, $(4, 7)$ and $(6, 8)$.

Since component 1 is assigned to $(H_1, H_2, H_3) = (0, 0, 0)$ we can deduce that %
\begin{equation*}
\begin{matrix}
    (H_1, H_2, H_3): & (0, 0, 0), &  (1, 0, 0), &  (0, 1, 0), &  (1, 1, 0), &  (0, 0, 1), &  (1, 0, 1), &  (0, 1, 1), &  (1, 1, 1)\\
    comp. \#: & 1 & 6 & 3 & ? & 5 & ? & ? & ?
    \end{matrix}
\end{equation*}

Assume that we know which components correspond to the hidden variable state $(H_1, H_2, H_3) = (h_1, h_2', h_3)$ and $(H_1, H_2, H_3) = (h_1', h_2, h_3)$, with $h_1\ne h_1'$ and $h_2\ne h_2'$. Then we can use the information above to deduce which components correspond to the hidden state $(h_1', h_2', h_3)$ since it differs from them in just 1 position. Hence, we can deduce
\begin{equation*}
\begin{matrix}
    (H_1, H_2, H_3): & (0, 0, 0), &  (1, 0, 0), &  (0, 1, 0), &  (1, 1, 0), &  (0, 0, 1), &  (1, 0, 1), &  (0, 1, 1), &  (1, 1, 1)\\
    comp. \#: & 1 & 6 & 3 & 7 & 5 & 8 & 2 & ?
    \end{matrix}
\end{equation*}
Note that since $(1,1,1)$ differs from the four states identified in the first step in two entries, this has not been determined yet.
However, repeating this argument a third time we can deduce that component 4 corresponds to $(H_1, H_2, H_3) = (1, 1, 1)$.

\end{example}

To illustrate how this algorithm works in the case of non-binary latent variables we provide one more example.

\begin{example}\label{example2}
Assume that $\pr(\ver)$ is Markov with respect to the DAG $\gr$ in Figure~\ref{fig:example:PHlearning-app-2} where we make no assumption about causal relation between $H_1$ and $H_2$. Assume that $\dim(H_1) = \dim(H_2) = 3$.

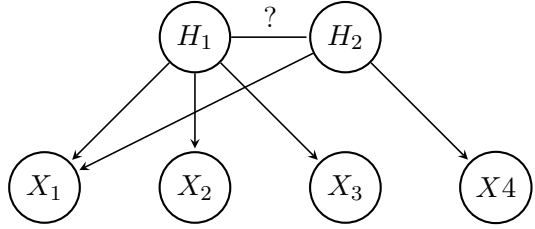
\begin{figure}
\begin{center}
\begin{tikzpicture}[
            > = stealth, %
            shorten > = 1pt, %
            auto,
            node distance = 2cm, %
            semithick %
        ]

        \tikzstyle{every state}=[
            draw = black,
            thick,
            fill = white,
            minimum size = 4mm
        ]
        \node[state] (h1) {$H_1$};
        \node[state] (h2) [right of=h1] {$H_2$};
        \node[state] (x2) [below  of=h1] {$X_2$};
        \node[state] (x3) [below  of=h2] {$X_3$};
        \node[state] (x4) [right of=x3] {$X4$};
        \node[state] (x1) [left of=x2] {$X_1$};

        \path[->] (h1) edge (x1);
        \path[->] (h1) edge (x2);
        \path[->] (h1) edge (x3);
        \path[->] (h2) edge (x1);
        \path[->] (h2) edge (x4);
        \path (h1) edge node[above]{?} (h2);
        
         \node [below right=.5cm, align=left,text width=8cm] at (x1)
        {
           
        };
\end{tikzpicture}
\end{center}
\caption{A bipartite graph $\Gamma$ in Example~\ref{example2}}\label{fig:example:PHlearning-app-2}
\end{figure}

Suppose that the map $L:[9]\rightarrow [9]\times [3]\times [3]\times [3]$ is given by: 
\begin{equation*}
\begin{matrix}
    i: &1 & 2 & 3 & 4 & 5\\ 
    L(i): & (1, 2, 1, 3), &  (3, 3, 3, 1), &  (4, 1, 2, 2), &  (2, 2, 1, 1), &  (7, 2, 1, 2), \\
    i: & 6 & 7 & 8 & 9 &\ \\
    L(i): &  (5, 1, 2, 1), &  (9, 1, 2, 3), &  (8, 3, 3, 3) & (6, 3, 3, 2) & \
    \end{matrix}
\end{equation*}
We want to find the correspondence between $h\in \Hdom = \{0, 1, 2\}^2$ and $i\in [9]$. 

As in the previous example, in order to see which components correspond to the states of latent variables where $H_2$ is fixed and $H_1$ takes all values in $\{0, 1, 2\}$ we group together the components that have the same value of $L$ on $X\setminus \ch(H_1) = \{X_4\}$. We get the following groups $(1, 7, 8)$, $(2, 4, 6)$ and $(3, 5, 9)$.

Similarly, by comparing the values of $L$ on $X\setminus \ch(H_2) = \{X_2, X_3\}$ we get that the following groups correspond to a fixed value of $H_1$, while $H_2$ vary: $(1, 4, 5)$, $(2, 8, 9)$ and $(3, 6, 7)$.

Since values of $H_i$ are determined up to relabeling we can arbitrarily assign a component, say 1, to $(H_1 = 0, H_2 = 0)$. Now, using Lemma~\ref{lem:hidden-directions}, we know that components that correspond to $(H_1 = 1, H_2 = 0)$ and $(H_1 = 2, H_2 = 0)$ are $7$ and $8$, and again because values of $H_i$ can be relabeled, at this point the choice is arbitrary. Using the similar argument for $H_2$, we can deduce the following correspondence:
\begin{equation*}
\begin{matrix}
    (H_1, H_2): & (0, 0), &  (1, 0), & (2, 0) &  (0, 1), &  (1, 1), &  (2, 1), &  (0, 2), &  (1, 2), &  (2, 2)\\
    comp. \#: & 1 & 7 & 8 & 4 & ? & ? & 5 & ? & ?
    \end{matrix}
\end{equation*}

At this point the labeling of the values of hidden variables is fixed. Now let us consider an index of hamming weight 2, say $(1, 1)$. We know that the component, that corresponds to this state of latent variables, differs from the component $4$, that corresponds to $(0, 1)$, only due to the change of $H_1$. Hence, the component that corresponds to $(1, 1)$ is in the set $\{2, 4, 6\}$. At the same time, we know that it differs from the component $7$ that corresponds to $(1, 0)$ only due to the change of $H_2$. Hence, the desired component is in the set $\{3, 6, 7\}$. By taking the intersection of sets $\{2, 4, 6\}$ and $\{3, 6, 7\}$ we deduce that the value that corresponds to $(1, 1)$ is 6. Similarly we can determine the rest of the values.

\begin{equation*}
\begin{matrix}
    (H_1, H_2): & (0, 0), &  (1, 0), & (2, 0) &  (0, 1), &  (1, 1), &  (2, 1), &  (0, 2), &  (1, 2), &  (2, 2)\\
    comp. \#: & 1 & 7 & 8 & 4 & 6 & 2 & 5 & 3 & 9
    \end{matrix}
\end{equation*}

\end{example}

\section{Implementation details}
\label{sec:alg}

The results in Section~\ref{sec:exact} assume access to the mixture oracle $\oracle(S)$. Of course, in practice, learning mixture models is a nontrivial problem. Fortunately, many algorithms exist for approximating this oracle:
In our implementation, we used $K$-means.
A na\"ive application of clustering algorithms, however, ignores the significant structure \emph{between} different subsets of observed variables. Thus, we also enforce internal consistency amongst these computations, which makes estimation much more robust in practice. In the remainder of this section, we describe the details of these computations; a complete outline of the entire pipeline can be
found in Appendix~\ref{app:pipeline}.

\paragraph{Estimating the number of marginal components}

In order to estimate the number of components in a marginal distribution for a subset $S$ of observed variables with $|S|\leq 3$, we use 
$K$-means combined with agglomerative clustering to merge nearby cluster centers,
and then select the number of components that has the highest silhouette score. Done independently, this step ignores the structure of the global mixture, and is not robust. In order to make learning more robust we observe that the assumptions on the distribution imply the following properties:
\begin{itemize}
    \item \textit{Divisibility condition:} The number of components we expect to observe over a set $S$ of observed variables is divisible by a number of components we observe on the subset $S'\subset S$ of observed variables (see Obs.~\ref{obs:counting}).
    \item \textit{Structure of means:} Observe that the projections of the means of mixture clusters in the marginal distribution over $S$ are the same as the means of mixture components over variables $S'$ for every $S'\subseteq S$. Hence, if we learn the mixture models over $S$ and $S'$ with the correct numbers of components $k(S)$ and $k(S')$, we expect the projections to be close. 
\end{itemize}

\begin{example}
Suppose we are confident that the number of components in the mixture over $X_1$ is in the set $\{6, 7, 8\}$, over $X_2$ is in $\{4, 5, 6\}$ and the number of components in the mixture over $\{X_1, X_2\}$ is in the set $\{20, 21, 22, 23, 24, 25, 26\}$. Using divisibility condition between $X_1$ and $\{X_1, X_2\}$ we may shrink the set of candidates to $\{21, 24\}$. Next using the divisibility condition for $X_2$ and $\{X_1, X_2\}$ we may determine that the number of components should be $24$.
\end{example}

With these observations in mind, we use a weighted voting procedure, where every set $S$ votes for the number of components in every superset and every subset based on divisibility or means alignment. We then  predict the true number of components by picking the candidate with the most votes.

\paragraph{Constructing $L$}

In order to estimate $L$ from samples we learn the mixture over the entire set of variables (using K-means and the number of components predicted on the previous step) and over each variable separately (again, using previous step). After this we project the mean of each component to a space over which $X_i$ is defined and pick the closest mean in $L_2$ distance (see Figure~\ref{fig:example:PHlearning}).

\paragraph{Reconstructing the latent graphical model}

Once we obtain the joint probability table of $H$, the final piece is to learn the latent DAG $\Lambda$ on $H$. This is a standard problem of learning the causal structure among $m$ discrete variables given samples from their joint distribution. For this a multitude of approaches have been proposed in the literature, for instance the PC algorithm \citep{spirtes1991} or the GES algorithm \cite{chickering2002optimal}. In our experiments, we use the Fast Greedy Equivalence Search \cite{ramsey2017million} with the discrete BIC score, without assuming faithfulness. The final graph $G$ is therefore obtained from $\Gamma$ and $\Lambda$.

\section{Experiments}
\label{sec:expts}

We implemented these algorithms in an end-to-end pipeline that inputs observed data and outputs an estimate of the causal graph $\gr$ and an estimate for the joint probability table $\prob(H)$.
To test this pipeline, we ran experiments on synthetic data. 
Full details about these experiments, including a detailed description of the entire pipeline, can be found in \cref{sec: expt_details}.

\paragraph{Data generation} We start with a causal DAG $G$ generated from the Erd\"{o}s-R\'{e}nyi model, for different settings of $m, n$ and $|\Omega_i|$. We then generate samples from the probability distribution that corresponds to $G$. We take each mixture component to be a Gaussian distribution with random mean and covariance (we do not force mixture components to be well-separated, aside from constraining the covariances to be small). Additionally, we do not impose restrictions on the weights of the components, which may be very small. As a result, it is common to have highly unbalanced clusters (e.g. we may have less than $30$ points in one component and over $1000$ in another). 
Figure~\ref{fig:shd} reports the results of $600$ simulations; $300$ each for $N=10000$ samples and $N=15000$ samples.

\paragraph{Results}
To compare how well our model recovers the underlying DAG, we compute the Structural Hamming Distance (SHD) between our estimated DAG and the true DAG. Since GES returns a CPDAG instead of a DAG, we also report the number of correct but unoriented edges in the estimated DAG. 
The average SHD across different problems sizes ranged from zero to $1.33$. The highest SHD for any single run was $6$. For context, the simulated DAGs had between $3$ and $25$ edges.
Note that any errors are entirely due to estimation error in the $K$-means implementation of $\oracle$, which we expect can be improved significantly. \camadd{In the supplement we also report on experiments with much smaller sample size $N = 1000$ (\cref{fig:shd-2})}. 
These results indicate that the proposed pipeline is surprisingly effective at recovering the causal graph.

\begin{figure}[t]
    \centering
    \includegraphics[width=\textwidth, height=\textheight, keepaspectratio]{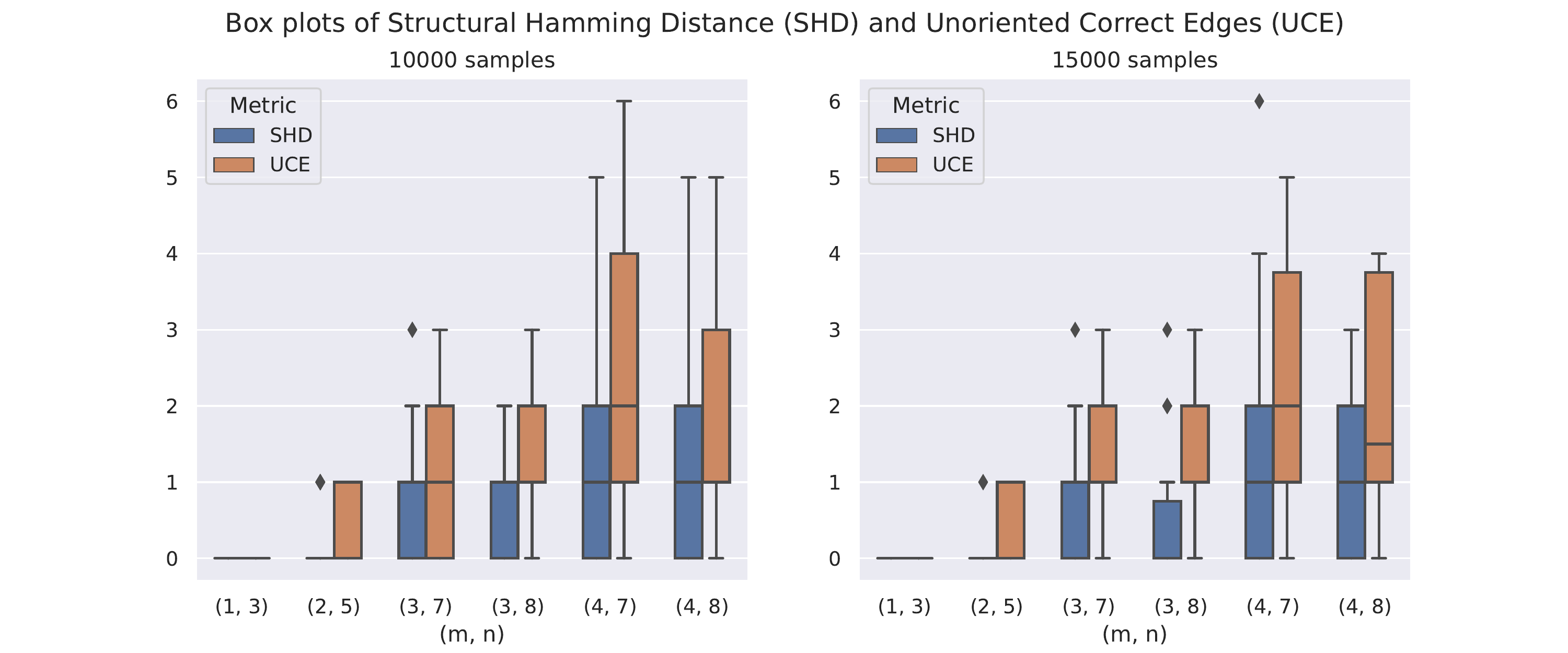}
    \caption{Average Structural Hamming distance for recovery of $\gr$, where $m=|\lat|$ and $n=|\obs|$.}
    \label{fig:shd}
\end{figure}

\section{Discussion}

\camadd{In this paper, we established general conditions under which the latent causal model $\gr$ is identifiable (Theorem~\ref{thm:main}). We show that these conditions are essentially necessary, and mostly amount to non-degeneracy conditions on the joint distribution. 
Under a linear independence condition on columns of the bipartite adjacency matrix of $\Gamma$, we propose a polynomial time algorithm for recovering $\Gamma$ and $\prob(H)$. Our algorithms work by reduction to the mixture oracle, which exists whenever the mixture model over $X$, naturally induced by discrete latent variables, is identifiable. Experimental results show effectiveness of our approach. Even though identifiability of mixture models is a long-studied problem, a good mixture oracle implementation is a bottleneck for scalability of our approach. We believe that it may be improved significantly, and consider this as a promising future direction. In this paper, we work under the measurement model that does not allow direct causal relationships between observed variables. We believe that this condition may be relaxed and are eager to explore this direction in future work. }  

\section{Acknowledgements}

G.R. thanks Aravindan Vijayaraghavan for pointers to useful references. \camadd{B.K. was partially supported by advisor L\'aszl\'o Babai's NSF grant  CCF 1718902. 
G.R. was partially supported by NSF grant CCF-1816372. 
P.R. was supported by NSF IIS-1955532.}
B.A. was supported by NSF IIS-1956330, NIH R01GM140467, and the Robert H. Topel Faculty Research Fund at the University of Chicago Booth School of Business.

\bibliography{latentdag}
\bibliographystyle{alpha}

\clearpage
\appendix

\section{Non-identifiability if Assumption~\ref{assump:strong} is violated}\label{app:example-nonidentifiable}

In this appendix we are going to show that Assumptions~\ref{assm:twins} and~\ref{assm:minimal} on the graph $G$ are not sufficient for identifiability, and therefore additional assumptions on the distribution of $H$ over $\Hdom$ are required as well.

\begin{definition}
For distributions $D_1, D_2$, let $D_1\otimes D_2$ denote the product of the distributions $D_1$ and $D_2$.
\end{definition}
That is, if $X\sim D_1$ and  $Y\sim D_2$ are independent, then their joint distribution is $D_1\otimes D_2$.

The following example illustrates an important case of non-identifiability and motivates the need for Assumption~\ref{assump:strong}.

\begin{example}
Let $N_0, N_1, N_0', N_1'$ be independent Gaussian distributions with distinct parameters (means and variances).  Consider  \begin{equation}\label{eq:non-iden-ex-app}
    (X_1, X_2)\sim \dfrac{1}{2}N_0\otimes N_0'+\dfrac{1}{4}N_1 \otimes  N_0'+\dfrac{1}{4}N_1\otimes  N_1'
\end{equation}
We claim that $(X_1, X_2)$ is consistent with (i.e., satisfies Markov property with respect to) each of the following three models below. Here, in the model $S_3$ the hidden variable $H_1$ can take three values $\{0, 1, 2\}$, and in models $A$ and $B$, hidden variables take values in $\{0, 1\}$.

\begin{center}
\begin{tikzpicture}[
            > = stealth, 
            shorten > = 1pt, 
            auto,
            node distance = 2cm, 
            semithick 
        ]

        \tikzstyle{every state}=[
            draw = black,
            thick,
            fill = white,
            minimum size = 4mm
        ]
        \node[state] (h1) {$H_1$};
        \node[state] (x1) [below of=h1] {$X_1$};
        \node[state] (h2) [right of=h1] {$H_2$};
        \node[state] (x2) [below  of=h2] {$X_2$};

        \path[->] (h1) edge (x1);
        \path[->] (h2) edge (x2);
        \path[->] (h1) edge (x2);
        
         \node [below right=.5cm, align=left,text width=8cm] at (x1)
        {
            Model $A$
        };

        \node[state] (hh1) [right of=h2]{$H_1$};
        \node[state] (xx1) [below of=hh1] {$X_1$};
        \node[state] (hh2) [right of=hh1] {$H_2$};
        \node[state] (xx2) [below  of=hh2] {$X_2$};
        \path[->] (hh1) edge (hh2);
        \path[->] (hh1) edge (xx1);
        \path[->] (hh2) edge (xx2);
        
         \node [below right=.5cm, align=left,text width=8cm] at (xx1)
        {
            Model $B$
        };
        
        \node[state] (x32) [left  of=x1] {$X_2$};
        \node[state] (x22) [left of= x32] {$X_1$};
        \node[state] (h12) [above of=x22] {$H_1$};
        
        \path[->] (h12) edge (x22);
        \path[->] (h12) edge (x32);

         \node [below right=0.5cm, align=left,text width=8cm] at (x22)
        {
            Model $S_3$
        };
\end{tikzpicture}
\end{center}

Note that all these models satisfy ``no-twins'' Assumption~\ref{assm:twins} and minimality Assumption~\ref{assm:minimal}, while Assumptions~\ref{assump:strong} are violated by models $A$ and $B$. 

\begin{enumerate}
    \item Consistency with $S_3$. Let $H_1$ be a random variable that takes values $0, 1, 2$ with probabilities $(1/2, 1/4, 1/4)$. Then
    \[ \begin{gathered}
    (X_1, X_2)\sim \sum\limits_{j\in \{0, 1, 2\}}\prob(X|H_1 = j)\prob(H_1 = j), \text{ where} \\
    \prob(X|H_1 = 0) = N_0\otimes N_0', \quad \prob(X|H_1 = 1) = N_1\otimes N_0', \quad \prob(X|H_1 = 2) = N_1\otimes N_1'
    \end{gathered}\]
    \item Consistency with $A$. Let $H_1$ and $H_2$ be i.i.d random variables that take values $0, 1$ with probabilities $(1/2, 1/2)$. Then
    \[ \begin{gathered}
    (X_1, X_2)\sim \sum\limits_{i\in \{0, 1\}}\sum\limits_{j\in \{0, 1\}}\prob(X|H_1 = i, H_2 = j)\prob(H_1 = i)\prob(H_2 = j), \text{ where} \\
    \prob(X|H_1 = 0) = N_0\otimes N_0', \quad \prob(X|H_1 = 1, H_2 = 0) = N_1\otimes N_0',\\ \prob(X|H_1 = 1, H_2 = 1) = N_1\otimes N_1'
    \end{gathered}\]
    \item Consistency with $B$. Let $H_1$ be a random variable that takes values ${0, 1}$ with probabilities $(1/2, 1/2)$. Let $H_2$ be a dependent random variable that takes values $0, 1$ with probabilities $(1, 0)$, if $H_1 = 0$, and with probabilities $(1/2, 1/2)$, if $H_1 = 1$.
    \[ \begin{gathered}
    (X_1, X_2)\sim \sum\limits_{i\in \{0, 1\}}\sum\limits_{j\in \{0, 1\}}\prob(X_1|H_1 = i)\prob(X_2|H_2 = j)\prob(H_1 = i)\prob(H_2 = j|H_1 = i), 
    \end{gathered}
\]
where
\[
\begin{gathered}
    \prob(X_1|H_1 = 0) = N_0, \quad \prob(X_1|H_1 = 1) = N_1,\\ 
    \prob(X_2|H_2 = 0) = N_0', \quad \prob(X_2|H_2 = 1) = N_1'
    \end{gathered}\]
\end{enumerate}

\end{example}

\begin{remark}
Observe that among the models $A, B$ and $S_3$, only $S_3$ satisfies Assumption~\ref{assump:strong}. Observe that the model $A$ satisfies part \ref{assump:strong:NZ}, but not \ref{assump:strong:SDC}, and the model $B$ satisfies part \ref{assump:strong:SDC}, but not \ref{assump:strong:NZ}, of Assumption~\ref{assump:strong}. This shows that only one of these assumptions is still not sufficient for identifiability of a latent causal model. 
\end{remark}

\section{Reconstructing bipartite part $\Gamma$. Proofs for Sections~\ref{sec:bipartite}}\label{app:bipartite}

Recall that (cf. Section~\ref{sec:recovery-ideas}), that for $w(H_i) = \log(\dim(H_i))$ and every subset $S\subseteq X$ the parameters of the latent DAG satisfy 
\begin{equation}
    \log(k(S)) = \sum\limits_{H_i\in \pa(S)} w(H_i).
\end{equation}
Recall also the definitions of $\com$ and $\comW$ in \eqref{eq:sne:Wsne}, reproduced here for ease of reference:
\begin{align*}
    \com_{\Gamma}(S) = \bigcap\limits_{x\in S} \nbhd_{\Gamma}(x) \quad \text{and}\quad  \comW_{\Gamma}(S) = \sum\limits_{v\in \com_{\Gamma}(S)} w(v).
\end{align*}

\subsection{Learning a bipartite graph with a hidden part from an additive score}\label{app:weighted-bipartite-learning}

We start our discussion of the proof of results in Section~\ref{sec:bipartite} by reducing learning of the causal graph $\Gamma$ to a more general learning problem.

Let $\Gamma = (X\cup H, E)$ be a  (not necessarily directed) bipartite graph on parts $X$ and $H$, and let $w:H\rightarrow (0, \infty)$ be an arbitrary function that defines weights of variables in $H$.

Recall that for a weight function $w$ and subset $S\subseteq X$ we define
\begin{equation}\label{eq:score-comp-app}
    W_{\Gamma}(S) = \sum\limits_{v \in \nbhd_{\Gamma}(S)} w(v)
\end{equation}

\begin{problem}\label{bip-problem}
Assume that the vertices in $H$ and the weight function $w$ are unknown. 
\item[]\quad  \textbf{Input:} Values $(W_{\Gamma}(S)\mid S\in \mathcal{F})$ indexed by a family of known subsets $\mathcal{F}\subseteq 2^{X}$
\item[]\quad \textbf{Goal:} Reconstruct the number of unknown vertices $H$, the graph $\Gamma$ between $H$ and $X$ (up to an isomorphism), and the weight function $w$ from the input. 
\end{problem}

Whether it is possible to reconstruct $\Gamma$ and $w$ from the input may depend on the family $\mathcal{F}$ or some additional assumptions about the structure of the graph $\Gamma$. To account for weights $w$, we slightly modify Definition~\ref{def:k-recoveroble} as follows:
\begin{definition}
We say that $(\Gamma, w)$ is $\mathcal{F}$-recoverable if $(\Gamma, w)$ can be uniquely recovered from $X$ and the sequence $(W_{\Gamma}(S)\mid S\in \mathcal{F})$.
\end{definition}

In the sequel, we use this modified definition.

The most natural regime is when $\mathcal{F}$ contains the sets whose size is bounded:

\begin{definition}
We say that $(\Gamma, w)$ is $t$-recoverable if $(\Gamma, w)$ is $\binom{X}{\leq t}$-recoverable, where $\binom{X}{\leq t}$ denotes the collection of subsets of $X$ of size at most $t$.
\end{definition}

\subsection{Reconstructing $\Gamma$ with full information about $W$}\label{sec:bip-recovery-full-app}

In this section we study Problem~\ref{bip-problem}, when full information about $W_\Gamma(\cdot)$ is provided, i.e. $\calF = 2^X$.

 Although the algorithm considered here will have exponential in $|X|$ runtime, it sheds light on the minimal theoretical assumptions we need for proving identifiability of $\Gam$. We will consider more efficient algorithms in later sections.

We start by proving Observation~\ref{obs:twins-gamma-norecovery}, which notes that if $\nbhd_{\Gamma}(H_i) = \nbhd_{\Gamma}(H_j)$ for $H_i\ne H_j$, then $(\Gamma, w)$ is not $2^{X}$-recoverable.  
\begin{proof}[Proof of Observation~\ref{obs:twins-gamma-norecovery}]
Consider the graph $\Gamma'$ obtained from $\Gamma$ by replacing $H_1$ and $H_2$ with a single variable $H^*$ and by connecting $H^*$ by an edge to all vertices in $X$ that are adjacent with $H_1$ or $H_2$ in $\Gamma$. Define $w(H^*) = w(H_1)+w(H_2)$. Then $W_{\Gamma}(S) = W_{\Gamma'}(S)$ for any $S\subseteq X$. 
\end{proof}

\begin{corollary}
Let $\mathcal{F}\subseteq 2^X$. If there is a pair of distinct variables $H_i, H_j\in H$ such that $\nbhd_{\Gamma}(H_1) = \nbhd_{\Gamma}(H_2)$, then $(\Gamma, w)$ is not $\mathcal{F}$-recoverable.
\end{corollary}

We now prove that in the case $\mathcal{F} = 2^X$, this is the only obstacle. We start by showing that certain neighborhoods of hidden variables can be identified using $\comW(\cdot)$. 

As explained in Section~\ref{sec:recovery-ideas}, in the case when $\nbhd(H_i)\not\subset \nbhd(H_j)$ for all $H_j$, we expect $\comW(\cdot)$ to have a clear ``signature'' of $H_i$. We make this intuition precise in the definition and lemma that follows.

\begin{definition}
We say that a set $S$ of observed variables $X$ is a \emph{maximal neighborhood block} if $\comW(S)\neq 0$, but for any superset $S'$ of $S$ we have $\comW(S') = 0$. 
\end{definition}

\begin{lemma}\label{lem:max-nb-blocks-app}
A set $S\subseteq X$ is a maximal neighborhood block if and only if there exists a hidden vertex $H_i\in H$ such that $\nbhd_{\Gamma}(H_i) = S$ and for any other $H_j\in H$ we have $S\not\subseteq \nbhd_{\Gamma}(H_j)$.
\end{lemma}
\begin{proof}
Assume that $S\subseteq X$ is a maximal neighborhood block. Since $\comW(S)>0$ the set of common neighbours $\com_{\Gamma}(S)$ is non-empty. If $\com_{\Gamma}(S)$ contains a hidden vertex $H_{j}$ that is connected to a vertex $x\notin S$ then, $H_{j}\in \com_{\Gamma}(S\cup \{x\})$, and $\comW_{\Gamma}(S\cup \{x\})\geq w(H_j) >0$ which contradicts the assumption that $S$ is a maximal neighborhood block. Therefore, for every $H_{j}$ in $\com_{\Gamma}(S)$, we have $\nbhd_{\Gamma}(H_j) \subset S$. Therefore, there exists a variable $H_i$ such that $\nbhd_{\Gamma}(H_i) = S$ and for any other $H_j\in H$ we have $S\not\subseteq \nbhd_{\Gamma}(H_j)$.

The opposite implication can be verified in a similar way.
\end{proof}

\begin{theorem}[Theorem~\ref{thm:hidden-bipartite}, part~\ref{thm:hidden-bipartite:exp}]\label{thm:hidden-bipartite-app}
Let $\Gamma$ be a bipartite graph with parts $X$ and $H$. Assume that no pair of vertices in $H$ has the same set of neighbours (in $X$). Then $\Gamma$ is $2^{X}$-recoverable.  
\end{theorem}

\begin{proof}
We prove the claim of the theorem by induction on $|H|$. The statement for the base case $|H| = 0$ immediately follows from the fact that $W(S) = 0$ for all $v\in X$ if and only if $|H| = 0$ since $w(\cdot)>0$. Assume that we proved the claim for all $\Gamma$ with $|H| = t$ that satisfy the assumptions of the theorem. Let $\Gamma$ be a graph with $|H| = t+1$ that satisfies the assumptions of the theorem. 

Using Lemma~\ref{lem:com-weight}, compute values $\comW_{\Gamma}(S)$ for every $S\subseteq X$. Using values of $\comW(\cdot)$ we can find a maximal neighborhood block $Y\subseteq X$. By Lemma~\ref{lem:max-nb-blocks-app}, there exists a hidden vertex $H_i$ such that $\{H_i\} = \com_{\Gamma}(Y)$. Note that $w(H_i) = \comW(Y)$. 

 Denote by $\Gamma'$ the graph obtained from $\Gamma$ by deleting $H_i$. 

Now we verify that $\Gamma'$ satisfies the assumptions of the theorem. There is nothing to check if the set of hidden vertices of $\Gamma'$ is empty. Assume that $\Gamma'$ has a non-empty set of hidden vertices. First, note that all hidden vertices in $\Gamma'$ still  have distinct sets of neighbors. Second, note that (cf. \eqref{eq:score-comp-app}) $W_{\Gamma'}(S) = W_{\Gamma}(S)$ if $S\cap Y = \emptyset$ (i.e. $H_i\notin \nbhd_{\Gamma}(S)$), and 
\[ W_{\Gamma'}(S) = W_{\Gamma}(S) - w(H_i) = W_{\Gamma}(S) - \comW_{\Gamma}(Y)\]
if $S\cap Y$ is not empty. Thus, we can compute $W_{\Gamma'}$ from the values of $W_{\Gamma}$.

By the induction hypothesis $(\Gamma', w|_{\Gamma'})$ is uniquely recoverable from $W_{\Gamma'}(S)$. Let $\Gamma^*$ be the graph obtained from $\Gamma'$ by adding a new variable $H_Y$ of weight $\comW_{\Gamma}(Y)$ and edges between $H_Y$ and $Y$.  Then $\Gamma^*$ is isomorphic to $\Gamma$, and so $\Gamma$ is $2^X$-recoverable.
\end{proof}

\subsection{Efficient $t$-recovery of $\Gamma$ for $t\geq 3$}\label{sec:bip-from-tensor-app}

The approach proposed in Appendix~\ref{sec:bip-recovery-full-app} is exponential in the number of observed variables in the worst case, since we need to compute the scores of all subsets of $X$. In this section, we show that with a mild additional assumption, there is an efficient algorithm to learn the bipartite graph between hidden and observed variables.

As before, let $\Gamma = (X\cup H, E)$ be the bipartite graph between hidden and observed variables. 

Recall, that we defined $A$ to be the $|X|\times |H|$ adjacency matrix of $\Gamma$ (with $0, 1$ entries) and $a_i$ to denote the $i$-th column of $A$.

For a sequence of indices $I = (i_1, i_2, \ldots, i_t)\subseteq [n]$ define
\begin{equation}
    \comW_{\Gamma}(I) = \sum\limits_{j\in H} w(j)\underbrace{(a_j)_{i_1}(a_j)_{i_2}\ldots (a_j)_{i_t}}_{t} = \Big(\sum\limits_{j\in H} w(j) \underbrace{a_j\otimes a_j \otimes \ldots \otimes a_j}_{t}\Big)_{(I)}.
\end{equation}

Recall, that as pointed out in Remark~\ref{rem:compute:scores},  for any $S\subseteq X$ with $|S|\leq t$ the value $\comW_{\Gamma}(S)$ can be computed from the $\{W_{\Gamma}(S)\mid S\subseteq X,\ |S|\leq t\}$ using \cref{lem:com-weight}. Therefore, we can make the following observation.

\begin{observation}\label{obs:wsne-comp}
All entries of the the tensor $M_t = \sum\limits_{j\in H} w(j) (\underbrace{a_j\otimes a_j \otimes \ldots \otimes a_j}_{t})$ can be computed as $M_t(I) = \comW_{\Gamma}(I)$ in $O(2^tn^t)$ time and space assuming access to $\{W_{\Gamma}(S)\mid S\subseteq X,\ |S|\leq t\}$. 
\end{observation}

For fixed $t$ this is a poly-time computation.
Furthermore, in the settings we consider in Secrion~\ref{sec:bipartite} the values of $W_{\Gamma}$ can be computed from $\oracle$ using Observation~\ref{obs:counting}.

Now we want to recover the vectors $a_j$ from $M_t$. Since $a_j$ are the columns of the adjacency matrix of $\Gamma$ this is equivalent to recovering the adjacency matrix of $\Gamma$ or $\Gamma$ itself up to an isomorphism.

\begin{definition}
For an order-$t$ tensor $M_t$ its rank is defined as the smallest $r$ such that $M_t$ can be written as
\begin{equation}
    M_t = \sum\limits_{j =1}^{r} c_j\bigotimes_{i=1}^{t} x_{j}^{(i)}.
\end{equation}
Such decomposition of $M$ with precisely $r$ components is called a minimum rank decomposition or a CP-decomposition. 
\end{definition}

\begin{lemma}\label{lem:trecovery:reform-app}
If the decomposition 
\[M_t = \sum\limits_{j\in H} w(j)\underbrace{a_j\otimes a_j \otimes \ldots \otimes a_j}_{t}\]
is the unique minimum rank decomposition, then $(\Gamma, w)$ is $t$-recoverable.
\end{lemma}
\begin{proof}
In order to recover $\Gamma$ and $w$ we compute $M_t$ using $\{W_{\Gamma}(S)\mid S\subseteq X,\ |S|\leq t\}$. Then $a_j$ and $w(j)$ can be uniquely (up to permutation) identified from minimum rank decomposition of $M$.
\end{proof}

The following simplified version of Kruskal's condition was proposed by Lovitz and Petrov.

\begin{theorem}[{\cite[Theorem~2]{lovitz2021generalization}}]\label{thm:lovitzpetrov}
Let $m\geq 2$ and $t\geq 3$ be integers. Let $V = V_1\otimes V_2 \otimes \ldots \otimes V_t$ be a multipartite vector space over a field $\mathbb{F}$ and let 
\[ \{x_j^{(1)}\otimes x_j^{(2)}\otimes \ldots \otimes x_j^{(t)}\mid j\in [m]\}\subset V\}\]
be a set of $m$ rank-1 (product) tensors. For a subset $S\subseteq [m]$ with $|S|\geq 2$ and $j\in [t]$ define 
\[ d_i(S) = \dim \Span \{x_{j}^{(i)}\mid j\in S\}.\] 
If $2|S|\leq \sum\limits_{i=1}^{t} (d_i(S)-1)+1$ for every such $S$, then 
\[ \sum\limits_{j\in [m]} x_j^{(1)}\otimes x_j^{(2)}\otimes \ldots \otimes x_j^{(t)}\] 
constitutes a unique minimal rank decomposition.
\end{theorem}

In our settings the sufficient condition for having the unique minimal rank decomposition takes the following form.

\begin{corollary}\label{cor:fromKruskal}
Assume that for every $S\subseteq H$ with $|S|\geq 2$ we have
\[ \dim \Span \{a_j \mid j\in S\}\geq \dfrac{2}{t}|S|+1,\]
then the decomposition $M_t = \sum\limits_{j\in H} w(j) \underbrace{a_j\otimes a_j \otimes \ldots \otimes a_j}_{t}$ is the unique minimum rank decomposition and so $(\Gamma, w)$ is $t$-recoverable.
\end{corollary}
\begin{proof}
Take $\mathbb{F} = \mathbb{R}$, then the result follows from \ref{thm:lovitzpetrov} for $x_j^{(1)} = w(j)a_j$ and $x_j^{(i)} = a_j$.
\end{proof}

\begin{proof}[Proof of Theorem~\ref{thm:hidden-bipartite} part ~\ref{thm:hidden-bipartite:eff}.] Follows by combining Corollary~\ref{cor:fromKruskal} and Lemma~\ref{lem:trecovery:reform-app}.
\end{proof}

Learning the components of the minimum rank decomposition is a very well-studied problem for which a variety of algorithms have been proposed in the literature (see the survey~\cite{vijayaraghavan2020efficient} or the book~\cite{moitra2014algorithmic}). 
We can use Jennrich's algorithm~\cite{harshman} (see also \cite{vijayaraghavan2020efficient, moitra2014algorithmic} and the references therein)
as an efficient algorithm with guarantees:

\begin{theorem}[Jennrich's algorithm~\cite{harshman}]\label{thm:jennrich}
Assume that the components of the tensor $\mathcal{T} = \sum\limits_{i=1}^{r} a_i\otimes b_i\otimes c_i$ satisfy the following conditions. The vectors $\{a_i \mid i\in [r]\}$ are linearly independent, the vectors $\{b_i \mid i\in [r]\}$ are linearly independent, and no pair of vectors $c_i$, $c_j$ is linearly dependent for $i\neq j$. Then the components of the tensor can be uniquely recovered in $O(n^3)$ space and $O(n^4)$ time. 
\end{theorem}
\begin{remark}
Note that if all vectors $a_i$ are linearly independent, then the assumptions of Corollary~\ref{cor:fromKruskal} are satisfied.
\end{remark}

\begin{remark}
A similar problem for $t$-recovery (for weighted hypergraphs) arose in a completely different context~\cite{anari2018smoothed}. While in both papers the problem is reduced to recovering the minimum rank decomposition of a carefully constructed tensor, we give better recovery guarantees for this problem by using more recent uniqueness guarantees~\cite{lovitz2021generalization}.
\end{remark}

\section{Reconstruction of the probability distribution on $H$. Proofs for Section~\ref{sec:Ph}}\label{app:proof:Ph}

In this section we discuss how one may reconstruct the hidden probability distribution on $\prob(H)$ from 
\begin{itemize}
    \item the bipartite graph $\Gamma$, and 
    \item the function $\compproj:[K]\to[k_{1}]\times\cdots\times[k_{n}]$, and 
    \item the mixture weights (probabilities) $\{\wgt{X}{i} \mid i\in [k(\obs)] \} = \{\prob(Z=i)\mid i\in [k(\obs)]\}$
\end{itemize}

\subsection{A key lemma}

Below we formulate the key lemma that allows us to relate the structure present in the map $L$ with the causal structure in $G$.

Given a state $H=(h_1,\ldots,h_m)$ and its corresponding component $P(X\given H_1=h_1,\ldots,H_m=h_m)$, we want to identify the components $P(X\given H_1=h_1',H_2=h_2,\ldots,H_m=h_m)$ that result from changing just the first hidden variable while keeping every other hidden variable fixed. The next lemma says that we can identify such components by looking into the distribution of the observed variables that are not children of $H_1$.

\begin{lemma}\label{lem:hidden-directions}
Let $H_i$ be a hidden variable and let $\comp{{X\setminus \nbhd_{\Gamma}(H_i)}}{j}$ be an arbitrary mixture component observed in a marginal mixture distribution over the variables in $X\setminus \nbhd_{\Gamma}(H_i)$.  Let $C(j_1), C(j_2), \ldots C(j_t)$ be all the mixture components in the distribution of $X$ whose marginal distribution over $X\setminus \nbhd_{\Gamma}(H_i)$ is equal to $C({X\setminus \nbhd_{\Gamma}(H_i)}, j)$. In other words, $L(j_s)_i = j$ for all $s \in [t]$. Then $t = \dim(H_i)$ and every $C(j_s)$ for $s\in [t]$ corresponds to a distinct value of $H_i$.  
\end{lemma}
\begin{proof}
Observe that Assumption~\ref{assum:ssc} implies that $\nbhd_{\Gamma}(X\setminus \nbhd_{\Gamma}(H_i)) = H\setminus \{H_i\}$. Therefore, by  Assumption~\ref{assump:strong}\ref{assump:strong:SDC}, $p(X\setminus \nbhd_{\Gamma}(H_i) \mid H = h_1)\sim p(X\setminus\nbhd_{\Gamma}(H_i) \mid H = h_2)$, if and only if $h_1$ and $h_2$ differ only in the value of $H_i$. 
\end{proof}

\subsection{Proof of Theorem~\ref{thm:Ph}}

The algorithm described in the previous examples can be used to prove Theorem~\ref{thm:Ph}.
For this, we present a general algorithm to recover the correspondence $\Hdom \ni h\leftrightarrow i\in [K]$ using Lemma~\ref{lem:hidden-directions}.

\begin{proof}[Proof of Theorem~\ref{thm:Ph}]
Without loss of generality, we may assume that $H_i$ takes values from $\Omega_i = \{0, 1, \ldots, \dim(H_i)-1\}$ for every $i$.

Recall that the Hamming weight of a vector is the number of non-zero coordinates of this vector. Denote by $\Omega^{(t)}$ the set of elements of $\Omega = \Omega_1\times \Omega_2\times  \ldots \times\Omega_k$ of the Hamming weight at most $t$.  

We start by recovering the entries of the tensor that correspond to the indicies in $\Omega^{(1)}$.

Let us pick an arbitrary mixture component $C$ that participates in the observed mixture model and let us put it in correspondence to $h = (0, 0, \ldots 0)$.  We assign the probability of observing $C$ to the cell $\jointtable(0, 0,\ldots, 0)$. 

Take any $i\in [m]$. Consider the set of $d(H_i)$ mixture components $\{C_{i, a} \mid a\in \Omega_i\}$, guaranteed by Lemma~\ref{lem:hidden-directions}, that have the same distribution as $C$ in coordinates $X\setminus \ch(H_i)$ (here we take arbitrary indexing by $a$). Assign $C_{i, a}$ to the vector $h_{i, a}\in \Omega^{(1)}$ of Hamming weight 1, that has unique non-zero value $a$ in coordinate $i$. And let $\jointtable(h_{i, a})$ be the probability of observing $C_{i, a}$. 

Next, we claim that the (valid) correspondence $\Hdom \ni h\leftrightarrow i\in [K]$ for $h\in \Omega^{(t)}$ can be uniquely extended to the (valid) correspondence $\Hdom \ni h\leftrightarrow i\in [K]$ for $h\in \Omega^{(t+1)}$ for any $t = 1, \ldots, m-1$.

Indeed, let $h\in \Omega^{(t+1)}$ and let $i$ and $j$ be a pair of distinct non-zero coordinates of $h$. Let $h_i$ and $h_j$ be the vectors obtained by changing the $i$-th and $j$-th coordinates of $h$ to 0. Let $C_i$ and $C_j$ be the mixture components that correspond to $h_i$ and $h_j$.

Using Lemma~\ref{lem:hidden-directions}, for $s \in \{i, j\}$ we can find a set $M_u$ of $\dim(H_u)$ mixture components that are equally distributed with $C_s$ over $X\setminus \nbhd_{\Gamma}(H_s)$. We put into correspondence with $h$ the unique component in the intersection of $M_i$ and $M_j$. 
We define $\jointtable(h)$ to be the probability of observing this component.
\end{proof}

Next we show that our algorithm works in time that is almost linear in the output size (recall that $K\geq 2^m$ and $K$ is the size of the output).

\begin{observation}\label{obs:Phruntime}
The algorithm described in Theorem~\ref{thm:Ph} works in $O((nm+\max_{i}k_i)K)$ time.
\end{observation}
\begin{proof}
First, the algorithm in  Theorem~\ref{thm:Ph} computes the equivalence classes of components that correspond to states of latent variables that differ just in the value of $H_j$. Having access to $\Gamma$ and $L$, computing these equivalence classes takes at most $O(nmK)$ time (for each of the $m$ hidden variables we need to compare vectors of values of $L$ of length $n$ for $K$ components).

Once these equivalence classes are computed, the algorithm in Theorem~\ref{thm:Ph} sequentially fills in the joint probability table. If the entries with indices of Hamming weight $t$ are filled in, in order to determine the value of a cell with an index of hamming weight $t+1$, we explore at most $2\max_{i\in [m]} k_i$ elements of the corresponding equivalence classes. Since eventually we explore all $K$ cells of the joint probability table, the total runtime of this phase is bounded by $O(\max_{i\in [m]} k_i)K$.
\end{proof}

\subsection{Non-identifiability if Assumption~\ref{assum:ssc} is violated}

Finally, we prove the impossibility claim in Remark~\ref{obs:ssc-assum-violated}.

\begin{proof}[Proof of Remark~\ref{obs:ssc-assum-violated}] We claim that if Assumption~\ref{assum:ssc} is violated, then $\prob(H)$ cannot be recovered and moreover $G$ is not identifiable. Consider a pair of models on Figure~\ref{example-ssc-violation}, where variables $H_1$ and $H_2$ are binary, i.e., they take values $\{0, 1\}$. Let $N_0, N_1, N_2, N_3$ and $N_0', N_1'$ be independent Gaussian distributions with distinct means and variances.

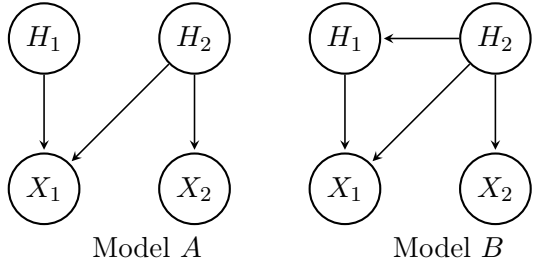
\begin{figure}[t]

\begin{center}
\begin{tikzpicture}[
            > = stealth, 
            shorten > = 1pt, 
            auto,
            node distance = 2cm, 
            semithick 
        ]

        \tikzstyle{every state}=[
            draw = black,
            thick,
            fill = white,
            minimum size = 4mm
        ]
        \node[state] (h1) {$H_1$};
        \node[state] (x1) [below of=h1] {$X_1$};
        \node[state] (h2) [right of=h1] {$H_2$};
        \node[state] (x2) [below  of=h2] {$X_2$};

        \path[->] (h1) edge (x1);
        \path[->] (h2) edge (x2);
        \path[->] (h2) edge (x1);
        
        \node [below right=.5cm, text width=8cm] at (x1)
        {
            Model $A$
        };

        \node[state] (hh1) [right of=h2]{$H_1$};
        \node[state] (xx1) [below of=hh1] {$X_1$};
        \node[state] (hh2) [right of=hh1] {$H_2$};
        \node[state] (xx2) [below  of=hh2] {$X_2$};
        \path[->] (hh2) edge (hh1);
        \path[->] (hh1) edge (xx1);
        \path[->] (hh2) edge (xx2);
        \path[->] (hh2) edge (xx1);
         \node [below right=.5cm, text width=8cm] at (xx1)
        {
            Model $B$
        };
\end{tikzpicture}
\end{center}
\caption{An example of the causal latent models that cannot be distinguished from observed data since Assumption~\ref{assum:ssc} is violated}\label{example-ssc-violation}
\end{figure}

Suppose that the observed distribution is equal to
\begin{equation}
    (X_1, X_2)\sim \dfrac{1}{9}N_0\otimes N_0'+\dfrac{2}{9}N_1\otimes N_1'+\dfrac{2}{9}N_2\otimes N_0'+\dfrac{4}{9}N_3\otimes N_1'
\end{equation}
Now we show that this distribution can be realized by both models A and B.
\begin{enumerate}
    \item \emph{Consistency with A}. Let $H_1, H_2$ be independent random variables that take values $\{0, 1\}$ with probabilities $(1/3, 2/3)$.
    \[ \begin{gathered}
    (X_1, X_2) \sim \sum\limits_{i\in \{0, 1\}}\sum\limits_{j\in \{0, 1\}}\prob(X_2|H_1 = i, H_2 = j)\prob(H_1 = i)\prob(H_2 = j), \text{ where} \\
    \prob(X_1|H_1 = 0, H_2 = 0) = N_0, \quad \prob(X_1|H_1 = 0, H_2 = 1) = N_1,\\
    \prob(X_1|H_1 = 1, H_2 = 0) = N_2, \quad \prob(X_1|H_1 = 1, H_2 = 1) = N_3,\\ 
    \prob(X_2|H_2 = 0) = N_0',\quad \prob(X_2|H_2 = 1) = N_1'
    \end{gathered}\]
    
    \item \emph{Consistency with B}. Let $H_1, H_2$ be binary random variables with the following distribution
    \begin{equation}
    \begin{gathered}
        \prob(H_2 = 0) = 1/3 \quad, \prob(H_1 = 0| H_2 = 0) = 1/3, \quad \prob(H_1 = 1| H_2 = 0) = 2/3,\\  
        \prob(H_2 = 0) = 2/3, \quad \prob(H_1 = 0| H_2 = 1) = 2/3,\quad  \prob(H_1 = 1| H_2 = 1) = 1/3
    \end{gathered}
    \end{equation}
    Define components of the mixture distribution to be
    \[ \begin{gathered}
    (X_1, X_2) \sim \sum\limits_{i\in \{0, 1\}}\sum\limits_{j\in \{0, 1\}}\prob(X_2|H_1 = i, H_2 = j)\prob(H_1 = i, H_2 = j), \text{ where} \\
    \prob(X_1|H_1 = 0, H_2 = 0) = N_0, \quad \prob(X_1|H_1 = 0, H_2 = 1) = N_3,\\
    \prob(X_1|H_1 = 1, H_2 = 0) = N_2, \quad \prob(X_1|H_1 = 1, H_2 = 1) = N_1,\\ 
    \prob(X_2|H_2 = 0) = N_0',\quad \prob(X_2|H_2 = 1) = N_1'
    \end{gathered}\]
\end{enumerate}
Since both models $A$ and $B$ realize distribution $\prob(X)$, we get that $G$ and $\prob(H)$ are not identifiable. Observe that Assumption~\ref{assum:ssc} is not satisfied for both $A$ and $B$, while Assumptions~\ref{assm:twins},~\ref{assm:minimal} and~\ref{assump:strong} are satisfied for each of $A$ and $B$.
\end{proof}

\section{Proof of Theorem~\ref{thm:main}}\label{app:proof:main}

Finally, we collect our results into a proof of the main theorem.

\begin{proof}[Proof of Theorem~\ref{thm:main}]
Suppose that Assumptions~\ref{assm:twins},~\ref{assm:minimal} and~\ref{assump:strong} hold, then by Theorem~\ref{thm:hidden-bipartite}\ref{thm:hidden-bipartite:exp}, $\Gamma$ and $\dim(H_i)$, for all $i$, can be recovered from $\prob(X)$. If additionally, the columns of the $|X|\times |H|$ adjacency matrix $A$ are linearly independent, then by Theorem~\ref{thm:recovery:eff} (see Corollary~\ref{cor:fromKruskal}, Theorem~\ref{thm:jennrich} and Observation~\ref{obs:wsne-comp}), $\Gamma$ and $\dim(H_i)$, for all $i$, can be reconstructed efficiently in $O(n^4)$ time.

Now, suppose that Assumption~\ref{assum:ssc} holds. We can extract the map $L$ from the $\oracle$ (by taking appropriate projections of component distributions). Therefore, since we have $\Gamma$, $\dim(H_i)$, $\{\wgt{X}{i}\}_{i\in [K]}$ and $L$, by Theorem~\ref{thm:Ph} and Observation~\ref{obs:Phruntime}, we can reconstruct $\prob(H)$ efficiently.
\end{proof}

\section{Algorithms}\label{app:pipeline}

In this section we describe the full pipeline\footnote{The code used to run the experiments can be found at \href{https://github.com/30bohdan/latent-dag}{https://github.com/30bohdan/latent-dag}} for learning $\gr$ from samples of the observed data $\obs$. As input we receive a set of samples and as output we return an estimated causal graph $G$ and a joint probability distribution over $H$. The pipeline consists of the following blocks:

\begin{enumerate}[label=(Step \alph*)]
    \item \textbf{Learning number of components.} 
    Estimates the number of components for all subsets of observed variables of size at most 3.
    \begin{itemize}
        \item Input: Samples from the distribution $\prob(X)$
        \item Output: Estimated number of mixture components $\ncomp(S)$ in $\prob(S)$ for all $S\subseteq X$, $|S|\leq 3$.
    \end{itemize}
    \item \textbf{Reconstruction of the bipartite graph.} Implements the algorithm of Theorem~\ref{thm:recovery:eff} for learning the bipartite causal graph $\Gamma$.
    \begin{itemize}
        \item Input: The number of mixture components $\ncomp(S)$ in $\prob(S)$ for all $S\subseteq X$, $|S|\leq 3$.
        \item Output: Estimated bipartite graph $\Gamma$ and sizes of the domains of hidden variables $\dim(H_i)$.
    \end{itemize}
    \item \textbf{Learning the projection map $L$.} 
     \begin{itemize}
        \item Input: Samples from the distribution $\prob(X)$ and the numbers of components $k(X)$ and $k(X_i)$ for every $i\in [n]$.
        \item Output: Estimated projection map $L$.
    \end{itemize}
    \item \textbf{Learning the distribution $\prob(H)$.} In this step we implement the algorithm described in Theorem~\ref{thm:Ph}, see also Algorithm~\ref{algo: ph learning}.
    \begin{itemize}
        \item Input: $L$, $\Gamma$ and $\dim(H_i)$ for all $i\in[m]$ and weights $\wgt{X}{j}$ of $k(X)$ mixture components.
        \item Output: Estimated joint probability table of $\prob(H)$.
    \end{itemize}
    We take $L$, $\Gamma$ and $\dim(H_i)$ for all $i$ as an input and return the joint probability table for $\prob(H)$ as an output.
    \item \textbf{Learning latent DAG $\latentgr$.}  In this step we estimate the causal graph over latent variables.
    \begin{itemize}
        \item Input: The joint probability table of $\prob(H)$.
        \item Output: Estimated causal graph $\Lambda$ over $H$.
    \end{itemize}
\end{enumerate}

In this paper, we prove theoretical guarantees for Steps (b) and (d), which invoke the mixture oracle $\oracle{}$. Step (a) implements $\oracle$, and Steps (c) and (e) are intermediate steps of the pipeline. As long as the oracle is correct, Step (c) is guaranteed to output the correct graph. The correctness of Step (e) depends on the structure learning algorithm used.
A nice feature of our algorithm is its modularity, if a better algorithm is developed for one of the steps, it can be incorporated without influencing other parts.

Below we discuss various implementation details for these steps.

\paragraph{Details of Step (a):}
Our implementation of Step (a) uses the following strategy.
\begin{enumerate}
    \item We estimate the upper bound $k_{max}$ on the number of components involved in the mixtures of single variables (this can be done using the silhouette score).
    \item For every observed variable $X_i$ we train $K$-means clustering  with $k = k_{max}$. After this, we perform agglomerative clustering for every $t\in [2, k_{max}]$, and record the silhouette score for every $t$.
    We pick $5$ values of $t$ with the best silhouette score.
    \item We use the divisibility condition to compute the sets $S_{X_i, X_j}$ of possible numbers of components we expect to see over the pairs of variables $X_i, X_j$. We use the best 5 predictions from the previous step for every variable $X_i$ and include the candidate for the number of components into  $S_{X_i, X_j}$ if it is divisible by one of the top-5 candidates for $X_i$ and for $X_j$. This step is mainly needed for computational purposes in order to restrict the number of candidates for the number of components observed over the pairs of variables. 
    \item Next we learn the mixture of $k$ components for every $k\in S_{X_i, X_j}$ over the pairs $(X_i, X_j)$ of observed variables. Similarly as in 2., we train $K$-means for the largest candidate and perform agglomerative clustering after that. 
    \item We use divisibility and means voting (discussed in Sec.~\ref{sec:alg}) to decide the best number of components for the single variables and the pairs of variables. In order to do this we make the predicted numbers of components for a pair $X_i, (X_i, X_j)$ to vote for each other if they satisfy the divisibility or means projection condition. We count the vote with the weight proportional to the silhouette score of the predicted number of components. For every $X_i$, and every pair $(X_i, X_j)$, we take the component with the largest amount of votes as our best prediction. 
    \item We use means of the components predicted for pairs of variables $(X_i, X_j)$ to estimate the locations of the means for the triples of observed variables. Instead of using $K$-means with the fresh start we initialize it with predicted locations. This improves the running time. We use $K$-means and silhouette score to predict the number of components for the triples of observed variables.
\end{enumerate}

\paragraph{Details of Step (b):}
In this step we use Corollary~\ref{cor:weight-observed-hidden}, Eq.~\eqref{eq:3tensor-formula} and Lemma~\ref{lem:com-weight} to compute entries of the tensor $M_3$ using the output of Step (a). After this we apply Jennrich's algorithm to learn the components of the tensor. As discussed in Appendix~\ref{sec:bip-from-tensor-app} this is sufficient to reconstruct $\Gamma$ and $\dim(H_i)$. 
In case Jennrich's algorithm did not successfully execute due to numerical issues, alternating least squares (ALS) was used as a failsafe. In this case, the number of hidden variables $m$ was used as input.\footnote{This can easily be avoided by running ALS for multiple values of $m$ and choosing the best fit. Since this issue arose in only a minority of cases, we did not implement this feature.}

\paragraph{Details of Step (c):}
We use $\Gamma$ and $\dim(H_i)$ to compute the number of components we expect to observe in $\prob(X_i)$ for every observed variable $X_i$ and the number of components in the distribution $\prob(X)$ over the entire set of observed variables. After this we use $K$-means to learn the components in the mixture distribution over every variable $X_i$ and over the entire set of observed variables. For every $i$, and for every mixture component of $\prob(X)$, we project its mean into the subspace over which $X_i$ is defined. We use the closest in $L_2$ distance mean of the components in $\prob(X_i)$ as a prediction for the projected component.

\paragraph{Details of Step (d):}

We implement the algorithm described in Theorem~\ref{thm:Ph}. See Algorithm~\ref{algo: ph learning} for details.

\begin{algorithm}[!ht]
\DontPrintSemicolon
\KwIn{ \begin{itemize}
    \item A bijective map $L : [k(X)] \rightarrow [k(X_1)]\times [k(X_2)]\times \ldots \times [k(X_n)]$;
    \item A bipartite graph $\Gamma$ between $X$ and $H$
    \item Values $\dim(H_i)$ for $i \in H$.
    \item Values $\prob(Z = i)$ for $i \in [k(X)]$ (the probabilities of observing the mixture components)
\end{itemize} }
\KwOut{An $\dim(H_1)\times \ldots \times \dim(H_m)$ tensor such that $\jointtable\cong \prob(H)$}
\tcp{Phase 1: use Lemma~\ref{lem:hidden-directions} to compute the sets of components that correspond to a change in a single hidden variable}
arrows = \{\}\\
\For{$H_i \in H$}{
S = $X\setminus ne_{\Gamma}(H_i)$\\
\For{ $c_1, c_2\in [k(X)]$}{ 
        \If{ $(L(c_2)_S == L(c_1)_S) $ and $c_1\neq c_2$}{
        arrows[$H_i$][$c_1$].append($c_2$)
        }
}
}
\tcp{Phase 2: initialize $T$ "along the edges"}
$A(0, \ldots 0) = 0$,\quad  $T(0, \ldots 0) = \prob(Z = 0)$\\
\For{$H_i \in H$ and $t\in \dim(H_i)$}{
$A(0, \ldots, t, \ldots 0) = arrows[H_i][0][t]$ \tcp{Note that an order does not matter}
$\jointtable(0, \ldots, t, \ldots 0) = \prob(Z = arrows[H_i][0][t])$
}
\tcp{Phase 3: reconstruct all other entries of the tensor}
$r = 1$\\
\While{$r<m$}{
    \For{$ind \in \dim(H_1)\times \ldots \dim(H_r)$}{
        \For{$j = r+1, \ldots, m$ and $t \in \dim(H_{t})$}{
            Let $i$ be the smallest index at which $ind$ is non-zero.\\
            Let $ind'$ be an index obtained from $ind$ by changing $j$-th entry from 0 to $t$ \\
            Let $ind''$ be obtained from $ind'$ by changing $i$-th entry to $0$.\\ 
            Let $x$ be the unique entry in the intersection of arrows[$H_i$][$A(ind'')$] and arrows[$H_{t}$][$A(ind)$].\\
            $A(ind') = x$\\
            $\jointtable(ind') = \prob(Z = x)$
        }
    }
}
\Return{$T$}
\caption{Learning $\prob(H)$}
\label{algo: ph learning}
\end{algorithm}

\paragraph{Details of Step (e):}
Once we obtain the estimated joint probability table, we run the Fast Greedy Equivalence Search \cite{ramsey2017million} to learn the edges of the Latent graph $\lat$, where we used the Discrete BIC score. FGES returns a CPDAG by default, so some edges may be undirected. We accordingly report both the Structural Hamming Distance (SHD) and the Unoriented Correct Edges (UCE) as metrics for our experiments. We remark that this step may be improved by using other algorithms such as PC \cite{spirtes1991} or other scores, which is an interesting direction for future work.

\section{Experiment details}\label{sec: expt_details}

\paragraph{Data generation}
For each experiment, 
the data generation process was as follows:
\begin{itemize}
    \item $(m, n)$: Chosen from among $(1, 3), (2, 5), (3, 7), (3, 8), (4, 7), (4, 8)$ in the ratio $1 : 2 : 2 : 3 : 1 : 1$
    \item Domain sizes $|\Omega_i|$: Sampled from $\{2, 3, 4, 5, 6\}$. If $|\Omega| = |\Omega_1| \ldots |\Omega_m| > 50$, we skip the experiment.
    \item $\pr(\lat)$: Generated via the Markov property. For each variable $\lat_i$, conditioned on its parents $\lat_{\pa(i)}$, a discrete distribution supported on $\Omega_i$ is chosen as follows: For each element $i$ in $\Omega_i$, a random integer $c_i$ is picked from $[1, 4]$ and distribution picks $i$ with probability proportional to $c_i$. 
    \item $\Lambda$: Choose an arbitrary topological order uniformly at random and sample each directed edge independently with probability $0.6$.
    \item $\Gamma$: Sample each directed edge from $\lat$ to $\obs$ with probability $0.5$. Enforce assumption~\ref{assum:ssc} and linear independence of the columns $a_j$ of the adjacency matrix $A$.
    \item Components: We generate Gaussian components for every $X_i$ in $\RR^5$ with random means and covariances. We take the means of the components to be sampled uniformly at random from the unit sphere. We take random symmetric diagonally dominant covariance matrices with the largest eigenvalue being $0.01$. (Note that for ~50 points on a unit 5-dimensional sphere, we expect to observe a pair of points at distance of the same order of magnitude). 
    \item Samples: We generate samples from the mixture components generated on the previous step with probabilities defined by $\prob(H)$.
\end{itemize}

We do not enforce minimum probability sizes or cluster sizes. As a result, the data generating process is likely to generate models which are extremely difficult to learn (e.g. if a randomly generated probability is very small, a mixture component will have few samples, which makes learning difficult). As a result, some random configurations may fail. We ran a total of $724$ experiments; out of these, 
 $8.3\%$ failed in the oracle learning phase and
 another $8.8\%$ failed to produce a graph because of very high domain sizes or unfeasible $L$. In the cases when the Jennrich algorithm failed due to numerical issues, this was caught and replaced with ALS for practical purposes as described in Step (b) above. These errors are conveniently caught during runtime and can be attributed to either the data generation process or the finite sample size as described above. \cref{fig:shd} reports the metrics for the remaining $600$ experiments: $300$ experiments each for $N = 10000$ samples and $N = 15000$ samples. 
 The experiments were run on a single node of an internal cluster.

\camadd{\paragraph{Experiments with smaller sample size.} The number of samples in the experiments discussed above is chosen so that every cluster component has approximately $20$ samples. We also explored the behaviour of our algorithms when the number of samples is much smaller. We ran a total of 136 experiments for $N = 1000$ samples, with $(m, n)$ chosen from $(1, 3), (2, 5), (3, 7), (4, 7), (3, 8)$ in proportion 1 : 2 : 1 : 1 : 1. Out of these, $4.4\%$ failed in the oracle learning phase and
 another $8.8\%$ failed to produce a graph because of very high domain sizes or unfeasible $L$. Furthermore, out of all failures, $25\%$ happen for $(m, n) = (4, 7)$ and  other $37.5\%$ happen for $(m, n) = (3, 8)$. We report the metrics on \cref{fig:shd-2}.
 
 \begin{figure}[t]
    \centering
    \includegraphics[width=0.99\textwidth, height=\textheight, keepaspectratio]{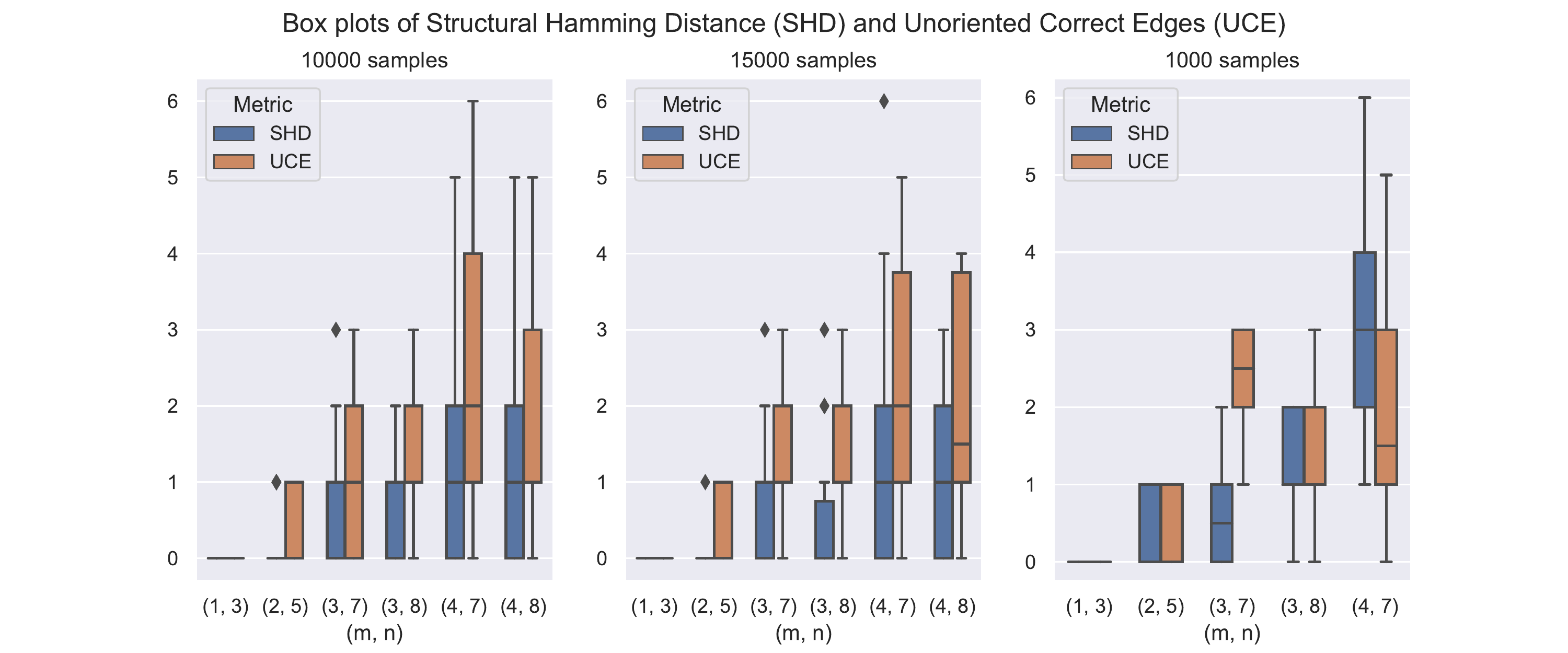}
    \caption{Average Structural Hamming distance for recovery of $\gr$, where $m=|\lat|$ and $n=|\obs|$.}
    \label{fig:shd-2}
\end{figure}
 
 We mention, that with $N = 1000$ samples, we were able to recover $H$ and $\Omega$ even in the cases when several latent states had fewer than five observations. Also, for comparison, to give an example where we were not able to recover $H$ and $\Omega$ exactly: the mixture model had 48 components with 1, 2, 2, 3, 3, 5, 5, 5, 6 …, 53, 55  samples per component. This is clearly an extremely challenging setup: Some states had only a few observations and the true number of components is unknown to the procedure.
}

\paragraph{Choice of parameters for learning $\Lambda$.} 
Once we have recovered the estimated joint probability table of $H$, to learn $\Lambda$, we use the Fast Greedy Equivalence Search algorithm \cite{ramsey2017million} with the Discrete BIC score. We use the PyCausal library \cite{chirayu_kong_wongchokprasitti_2019_3592985}. We used the default parameters (no hyperparameter tuning) and in particular, we did not assume faithfulness.
\clearpage
\paragraph{Approximate Runtime}

The average runtimes for each experiment are in the following table.
\captionof{table}{Average runtime in seconds}
\begin{center}
  \renewcommand{\arraystretch}{1.2}
  \begin{tabular}{|c|cc|}
  \hline
  \textbf{(m, n)} & \textbf{10000 samples} & \textbf{15000 samples}\\
    \hline
    (1, 3) & 30.64 s & 53.06 s \\
    (2, 5) & 89.03 s & 148.81 s \\
    (3, 7) & 288.25 s & 385.27 s \\
    (3, 8) & 320.25 s & 616.86 s \\
    (4, 7) & 297.32 s & 400.04 s \\
    (4, 8) & 361.28 s & 604.14 s \\
    \hline
\end{tabular}
\end{center}

\paragraph{Average number of edges} For our experiments, the average total number of edges in $\Lambda, \Gamma$ (also known as NNZ of $G$) are in the following table.

\captionof{table}{Average number of edges for different settings}
\begin{center}
  \renewcommand{\arraystretch}{1.2}
  \begin{tabular}{|c|cc|}
  \hline
  \textbf{(m, n)} & \textbf{Number of Samples} & \textbf{Average number of edges in $G = (\Lambda, \Gamma)$}\\
    \hline
    (1, 3) & 10000 & 3.0 \\
(1, 3) & 15000 & 3.0 \\
(1, 3) & 1000 & 3.0 \\
(2, 5) & 10000 & 7.15 \\
(2, 5) & 15000 & 6.95 \\
(2, 5) & 1000 & 6.98 \\
(3, 7) & 10000 & 13.52 \\
(3, 7) & 15000 & 13.2 \\
(3, 7) & 1000 & 13.7 \\
(3, 8) & 10000 & 15.27 \\
(3, 8) & 15000 & 15.16 \\
(3, 8) & 1000 & 15.3 \\
(4, 7) & 10000 & 17.43 \\
(4, 7) & 15000 & 18.17 \\
(4, 7) & 1000 & 18.35 \\
(4, 8) & 10000 & 19.87 \\
(4, 8) & 15000 & 20.13 \\
    \hline
\end{tabular}
\end{center}

\paragraph{Scatter plots} The scatter plots for the Structural Hamming distance (SHD) versus the total number of edges $|E(G)|$ in $G$ and that of the unoriented correct edges (UCE) vs $|E(G)|$ is given in \cref{fig:scatter}.

\begin{figure}[!h]
    \centering
    \includegraphics[width=\textwidth, height=\textheight, keepaspectratio]
    {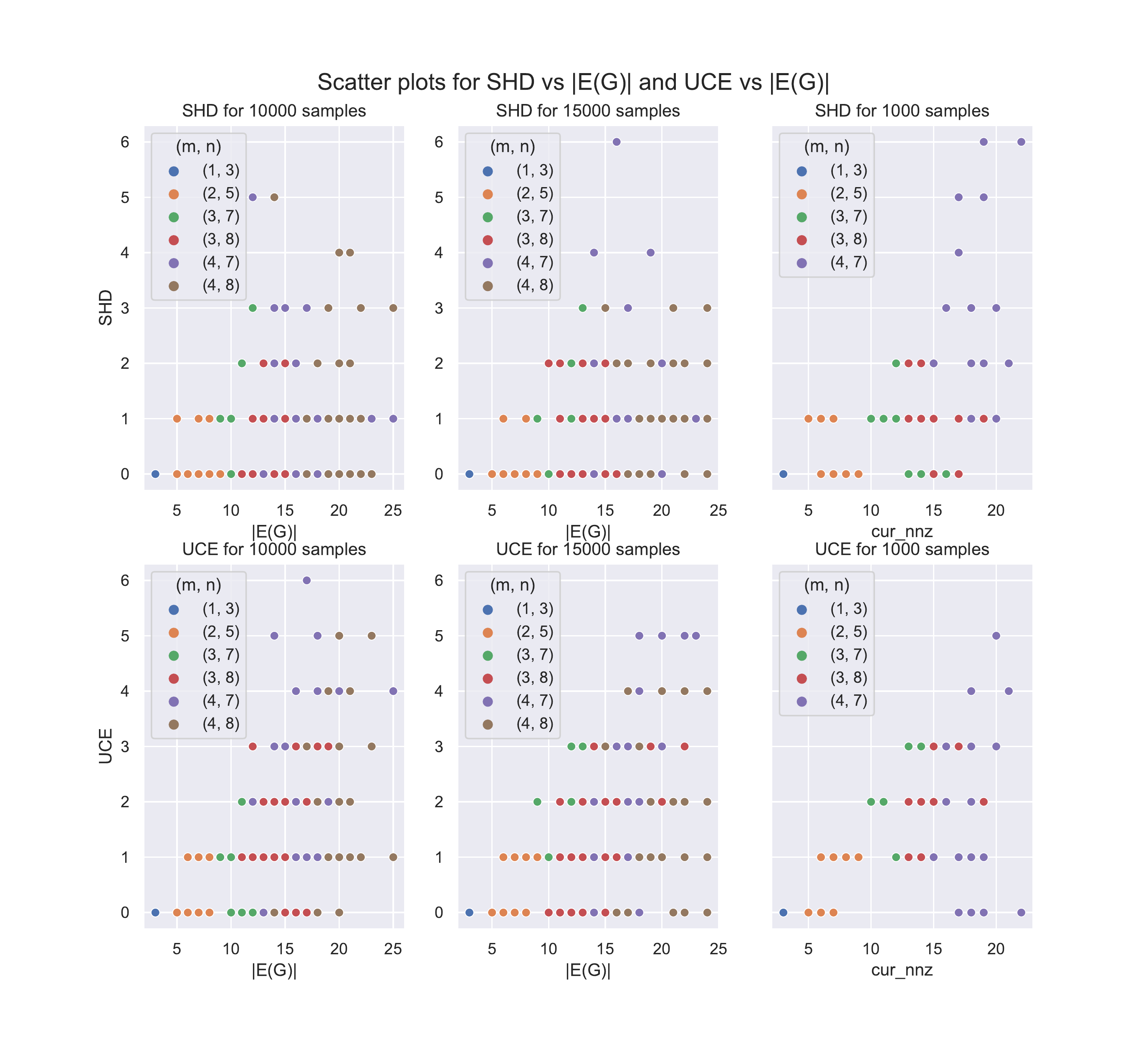}
    \caption{Scatterplots where $m=|\lat|$ and $n=|\obs|$.}
    \label{fig:scatter}
\end{figure}

\end{document}